\documentclass[11pt]{article}
\usepackage{amsmath,amssymb,amsfonts,amsthm,epsfig}
\usepackage[usenames,dvipsnames]{xcolor}

\usepackage{bm,xspace}

\usepackage{cancel}
\usepackage{fullpage}
\usepackage{liyang}
\usepackage{framed}
\usepackage{verbatim}
\usepackage{enumitem}
\usepackage{array}
\usepackage{multirow}
\usepackage{afterpage}
\usepackage{mathrsfs}

\usepackage[normalem]{ulem}
\usepackage{ytableau}
\usepackage{genyoungtabkitz}
\usepackage{young}

\usepackage{caption}

\usepackage{todonotes}

\makeatletter
\newtheorem*{rep@theorem}{\rep@title}
\newcommand{\newreptheorem}[2]{
\newenvironment{rep#1}[1]{
 \def\rep@title{#2 \ref{##1}}
 \begin{rep@theorem}\itshape}
 {\end{rep@theorem}}}
\makeatother
\theoremstyle{plain}

\makeatletter
\newenvironment{proofof}[1]{\par
  \pushQED{\qed}%
  \normalfont \topsep6\p@\@plus6\p@\relax
  \trivlist
  \item[\hskip\labelsep
\emph{    Proof of #1\@addpunct{.}}]\ignorespaces
}{%
  \popQED\endtrivlist\@endpefalse
}
\makeatother

\newcommand{\ignore}[1]{}

\def\colorful{1}

\ifnum\colorful=1

\newcommand{\red}[1]{{\color{red} {#1}}}

\newcommand{\gray}[1]{{\color{gray}{#1}}}

\fi
\ifnum\colorful=0

\newcommand{\red}[1]{{{#1}}}

\newcommand{\gray}[1]{{{#1}}}

\fi

\usepackage{boxedminipage}

\newreptheorem{theorem}{Theorem}
\newtheorem*{theorem*}{Theorem}
\newreptheorem{lemma}{Lemma}
\newreptheorem{proposition}{Proposition}
\newtheorem*{noclaim*}{Claim}

\newcommand{\cycles}{\mathrm{cycles}}

\newcommand{\marg}{\mathrm{marginal}}
\newcommand{\learn}{\mathrm{learn}}

\newcommand{\Ind}{\mathds{1}}

\newcommand{\coord}{\mathrm{Coord}}
\newcommand{\constr}{\mathrm{Constr}}

\usepackage{chngpage}

\usepackage{dsfont}
\renewcommand{\F}{\mathds{F}}
\renewcommand{\N}{\mathds{N}}
\renewcommand{\R}{\mathds{R}}

\newcommand{\paths}{\mathrm{Paths}}

\newcommand{\trans}{{\mathrm{trans}}}

\newcommand{\TInf}{\mathbf{I}}   
\newcommand{\maj}{\mathrm{Maj}}         
\newcommand{\ol}[1]{\overline{#1}}

\newcommand{\rnote}[1]{\footnote{{\bf \color{red}Rocco}: {#1}}}

\newcommand{\partn}[2]{\mathsf{Partial}_{{#1}}(#2)}

\begin{document}

\title{
Learning sparse mixtures of rankings\\
from noisy\ignore{and partial} information
}
\author{
Anindya De\thanks{Supported by NSF grant CCF-1814706}\\
Northwestern University\\
{\tt anindya@eecs.northwestern.edu}
\and
Ryan O'Donnell\thanks{Supported by NSF grants CCF-1618679 and CCF-1717606}\\
Carnegie Mellon University\\
{\tt odonnell@cs.cmu.edu}
\and
\and Rocco A.~Servedio\thanks{Supported by NSF grants CCF-1563155 and CCF-1814873 and by the Simons Collaboration on Algorithms and Geometry. This material is based upon work supported by the National Science Foundation under grant numbers listed above. Any opinions, findings and conclusions or recommendations expressed in this material are those of the authors and do not necessarily reflect the views of the National Science Foundation.}\\
Columbia University \\
{\tt rocco@cs.columbia.edu}
}

\begin{titlepage}

\maketitle

\begin{abstract}

We study the problem of learning an unknown mixture of $k$ rankings over $n$ elements, given access to noisy\ignore{or incomplete} samples drawn from the unknown mixture.  We consider a range of different noise\ignore{ and deletion} models, including natural variants of the ``heat kernel'' noise framework and the Mallows model.  For each of these noise models we give an algorithm which, under mild assumptions, learns the unknown mixture to high accuracy and runs in $n^{O(\log k)}$ time.
The best previous algorithms for closely related problems have running times which are exponential in $k$.
\end{abstract}

\thispagestyle{empty}

\end{titlepage}


\section{Introduction}

This paper considers the following natural scenario:  there is a large heterogeneous population which consists of $k$ disjoint subgroups, and for each subgroup there is a ``central preference order'' specifying a ranking over a fixed set of $n$ items (equivalently, specifying a permutation in the symmetric group $\mathbb{S}_n$). For each $i \in \{1,\dots,k\}$, the preference order of each individual in subgroup $i$ is assumed to be a noisy version of the central preference order (the permutation corresponding to subgroup $i$).  A natural learning task  which arises in this scenario is the following:  given access to the preference order of randomly selected members of the population, is it possible to learn the central preference orders of the $k$ sub-populations, as well as the relative sizes of these $k$ sub-populations within the overall population?

Worst-case formulations of the above problem typically tend to be (difficult) variants of the feedback arc set problem, which is known to be NP-complete \cite{GareyJohnson79}.  In view of the practical importance of problems of this sort, though, there has been considerable recent research interest in studying various generative models corresponding to the above scenario (we discuss some of the recent work which is most closely related to our results in~\Cref{sec:priorwork}).  In this paper we will model the above general problem schema as follows:  The $k$ ``central preference orders'' of the subgroups are given by $k$ unknown permutations $\sigma_1, \ldots, \sigma_k \in \mathbb{S}_n$.  The fraction of the population belonging to the $i$-th subgroup, for $1 \leq i \leq k$, is given by an unknown $w_i \geq 0$ (so $w_1 + \cdots + w_k = 1$). Finally, the noise is modeled by some family of distributions $\{\mathcal{K}_{\theta}\}$, where each distribution ${\cal K}_\theta$ is supported on $\mathbb{S}_n$, and the preference order of a random individual in the $i$-th subgroup is given by $\bpi \sigma_i$, where $\bpi \sim \mathcal{K}_\theta$. Here $\theta $ is a model parameter capturing the ``noise rate'' (we will have much more to say about this for each of the specific noise models we consider below).  The  learning task is to recover the central rankings $\sigma_1, \ldots, \sigma_k$ and their proportions $w_1, \ldots, w_k$, given access to preference orders of randomly chosen individuals from the population. In other words, each sample provided to the learner is independently generated  by first choosing a random permutation $\bsigma$, where $\bsigma$ is chosen to be $\sigma_i$ with probability $w_i$; then independently drawing a random $\bpi \sim {\cal K}_\theta$; and finally, providing the learner with the permutation $\bpi \bsigma \in \mathbb{S}_n.$ 
 Let $f: \mathbb{S}_n \rightarrow \mathbb{R}^{\geq 0}$ denote
the function which is $w_i$ at $\sigma_i$ and $0$ otherwise. 
With this notation, 
we write ``${\cal K}_\theta \ast f$'' to denote the distribution over noisy samples described above, and our goal is to approximately recover $f$ given such noisy samples. The reader may verify that the distribution defined by $\bpi \bsigma$ is precisely given by the group convolution ${\cal K}_\theta \ast f$   (and hence the notation).

\subsection{The noise\ignore{ (and deletion)} models that we consider}

We consider a range of different noise models, corresponding to different choices for the parametric family $\{{\cal K}_\theta\}$, and for each one we give an efficient algorithm for recovering the population in the presence of that kind of noise.  In this subsection we detail the three specific noise models that we will work with (though as we discuss later, our general mode of analysis could be applied to other noise models as well).

\medskip

(A.) {\bf Symmetric noise.} In the \emph{symmetric noise} model, the parametric family of distributions over $\mathbb{S}_n$ is denoted $\{{\cal S}_{\ol{p}}\}_{\ol{p} \in \Delta^{n}}$.  Given a vector $\ol{p}=(p_0,\dots,p_n) \in \Delta^n$ (so each $p_i \geq 0$ and $\sum_{i=0}^n p_i=1$), a draw of $\bpi \sim {\cal S}_{\ol{p}}$ is obtained as follows:

\begin{enumerate}
\item Choose $0 \le \bj \le n$, where value $j$ is chosen with probability ${p}_j$.
\item Choose a uniformly random subset $\bA \subseteq [n]$ of size exactly $\bj$. Draw $\bpi$ uniformly from $\mathbb{S}_{\bA}$; in other words, $\bpi$ is a uniformly random permutation over the set $\mathcal{A}$ and is the identity permutation on elements in $[n]\setminus \bA$.  (We denote this uniform distribution over $\mathbb{S}_{\bA}$ by $\mathbb{U}_{\bA}$.)
\end{enumerate}

\noindent Note that in this model, if the noise vector $\ol{p}$ has $p_n=1$, then every draw from $ {\cal S}_{\ol{p}} \ast f$ is a uniform random permutation and there is no useful information available to the learner.

\medskip

In order to define the next two noise models that we consider, let us recall the notion of a \emph{right-invariant} metric on $\mathbb{S}_n$. Such a metric $d(\cdot,\cdot)$ is one that satisfies $d(\sigma,\pi)=d(\sigma \tau, \pi \tau)$ for all $\sigma,\pi,\tau \in \mathbb{S}_n$.  We note that a metric is right-invariant if and only if it is invariant under relabeling of the items $1,\dots,n$, and that most metrics considered in the literature satisfy this condition (see~\cite{kumar2010generalized, diaconis-chap6} for discussions of this point).  In this paper, for technical convenience we restrict our attention to the metric $d(\cdot,\cdot)$ being the \emph{Cayley distance} over $\mathbb{S}_n$ (though see~\Cref{sec:discussion} for a discussion of how our methods and results could potentially be generalized to other right-invariant metrics):

\begin{definition}~\label{def:Cayley-distance}
Let $G$ be the undirected graph with vertex set $\mathbb{S}_n$ and an edge between permutations
$\sigma$ and $\pi$ if there is a transposition $\tau$ such that $\sigma = \tau \cdot \pi$. The \emph{Cayley distance} over $\mathbb{S}_n$ is the metric induced by this graph; in other words, $d(\pi,\sigma)=t$ where $t$ is the smallest value such that there are transpositions $\tau_1, \ldots, \tau_t$ satisfying $\sigma = \tau_1 \cdots \tau_t \pi$.
\end{definition}

Now we are ready to define the next two parameterized families of noise distributions that we consider.  We note that each of the noise distributions ${\cal K}$ considered below has the natural property that $\Pr_{\bpi \sim {\cal K}}[\bpi=\pi]$ decreases with $d(\pi,e)$ where $e$ is the identity distribution.

\medskip

(B.) {\bf Heat kernel random walk under Cayley distance.}  Let ${\cal L}$ be the Laplacian of the graph $G$ from~\Cref{def:Cayley-distance}.  Given a ``temperature'' parameter $t \in \R^+$, the \emph{heat kernel} is the $n! \times n!$ matrix $H_t = e^{-t{\cal L}}.$  It is well known that $H_t$ is the transition matrix of the random walk induced by choosing a Poisson-distributed time parameter $\bT \sim \mathsf{Poi}(t)$ and then taking $\bT$ steps of a uniform random walk in the graph $G$.  With this motivation, we define the \emph{heat kernel noise model} as follows:  the parametric family of distributions is $\{{\cal H}_t\}_{t \in \R^+}$, where the probability weight that ${\cal H}_t$ assigns to permutation $\pi$ is the probability that the above-described random walk, starting at the identity permutation $e \in \mathbb{S}_n$, reaches $\pi$.  (Observe that higher temperature parameters $t$ correspond to higher rates of noise.  More precisely, it is well known that the mixing time of a uniform random walk on $G$ is $\Theta(n \log n)$ steps, so if $t$ grows larger than $n \log n$ then the distribution $\calH_t$ converges rapidly to the uniform distribution on $\mathbb{S}_n$; see  \cite{DS81} for detailed results along these lines.)  We note that these probability distributions (or more precisely, the associated heat kernel $H_t$) have been previously studied in the context of learning rankings, see e.g.~\cite{kondor2002diffusion, KB10, jiao2018kendall}.  In some of this work, a different underlying distance measure was used over $\mathbb{S}_n$ rather than the Cayley distance; see our discussion of related work in~\Cref{sec:priorwork}.

\medskip (C.) {\bf Mallows-type model under Cayley distance (Cayley-Mallows / Ewens model).} While the heat kernel noise model arises naturally from an analyst's perspective, a somewhat different model, called the \emph{Mallows model}, has been more popular in the statistics and machine learning literature.  The Mallows model is defined using the ``Kendall $\tau$-distance'' $K(\cdot,\cdot)$ between permutations (defined in~\Cref{sec:priorwork}) rather than the Cayley distance $d(\cdot,\cdot)$; the Mallows model with parameter $\theta > 0$ assigns probability weight $e^{-\theta K(\pi,e)}/Z_K(\theta)$ to the permutation $\pi$, where $Z_k(\theta)=\sum_{\pi \in \mathbb{S}_n} e^{-\theta K(\pi,e)}$ is a normalizing constant. As proposed by Fligner and Verducci \cite{fligner1986distance}, it is natural to consider generalizations of the Mallows model in which other distance measures take the place of the Kendall $\tau$-distance.  The model which we consider is one in which the Cayley distance is used as the distance measure; so given $\theta>0$, the noise distribution ${\cal M}_\theta$ which we consider assigns weight $e^{-\theta d(\pi,e)}/Z(\theta)$ to each permutation $\pi \in \mathbb{S}_n$, where $Z(\theta) = \sum_{\pi \in \mathbb{S}_n} e^{-\theta d(\pi,e)}$ is a normalizing constant.  In fact, this noise model was already proposed in 1972 by W. Ewens in the context of population genetics \cite{Ewens72} and has been intensively studied in that field (we note that \cite{Ewens72} has been cited more than 2000 times according to Google Scholar).  To align our terminology with the strand of research in machine learning and theoretical computer science which deals with the Mallows model, in the rest of this paper we  refer to ${\cal M}_\theta$ as the \emph{Cayley-Mallows} model. For the same reason, we will also refer to the usual Mallows model (with the Kendall $\tau$-disance) as the \emph{Kendall-Mallows} model.\ignore{\rnote{I also propose we take the ``usual'' notation for Ewens, which is that you're proportional to $\theta^{\text{cyc}(\pi)}$, where cyc denotes the number of cycles.  This kinda fixes the increasing/decreasing thing.  Also, we can explicitly state the normalizing constant, $Z(\theta)$ in this case; it's $\theta^{\uparrow n} = \theta(\theta+1)\cdots(\theta+n-1)$.}} We observe that for the Cayley-Mallows model ${\cal M}_\theta$, in contrast with the heat kernel noise model now \emph{smaller} values of $\theta$ correspond to higher levels of noise, and that when $\theta=0$ the distribution ${\cal M}_\theta$ is simply the uniform distribution over $\mathbb{S}_n$ and there is no useful information available to the learner.

\subsection{Our results}

For each of the noise\ignore{ and deletion} models defined above, we give algorithms which, under a mild technical assumption (that no mixing weight $w_i$ is too small), provably recover the unknown central rankings $\sigma_1,\dots,\sigma_k$ and associated mixing weights $w_1,\dots,w_k$ up to high accuracy.  A notable feature of our results is that the sample and running time dependence is only \emph{quasipolynomial} in the number of elements $n$ and the number of sub-populations $k$; as we detail in~\Cref{sec:priorwork} below, this is in contrast with recent results for similar problems in which the dependence on $k$ is exponential.

Below we give detailed statements of our results.  The following notation and terminology will be used in these statements:  for $f$ a distribution over $\mathbb{S}_n$ (or any function from $\mathbb{S}_n$ to $\R$) we write $\supp(f)$ to denote the set of permutations $\sigma \in \mathbb{S}_n$ that have $f(\sigma) \neq 0$.  For a given noise model $\cal K$, we write ``$ {\cal K} \ast f$'' to denote the distribution over noisy samples that is provided to the learning algorithm as described earlier.  Given two functions $f,g: \mathbb{S}_n \to \R$, we write ``$\|f-g\|_1$'' to denote $\sum_{\pi \in \mathbb{S}_n} |f(\pi) - g(\pi)|$, the $\ell_1$ distance between $f$ and $g$. If $f$ and $g$ are both distributions then we write $\dtv(f,g)$ to denote the \emph{total variation distance} between $f$ and $g$, which is ${\frac 1 2} \|f-g\|_1.$  Finally, if $f$ is a distribution over $\mathbb{S}_n$ in which $f(\sigma)>\eps$ for every $\sigma$ such that $f(\sigma)>0$, we say that $f$ is \emph{$\eps$-heavy}.

\medskip
\noindent {\bf Learning from noisy rankings:  Positive and negative results.} Our first algorithmic result is for the symmetric noise model (A) defined earlier.  \Cref{thm:ub-symmetric}, stated below, gives an efficient algorithm as long as the vector $\ol{p}$ is ``not too extreme'' (i.e.~not too biased towards putting almost all of its weight on large values very close to $n$):

\begin{theorem} [Algorithm for symmetric noise] \label{thm:ub-symmetric}
There is an algorithm with the following guarantee: Let
$f$ be an unknown $\eps$-heavy distribution over $\mathbb{S}_n$ with $|\supp(f)| \leq k$. Let $\ol{p}=(p_0,\dots,p_n) \in \Delta^n$ be such that
\[
\sum_{j=0}^{n - \log k} p_j \geq {\frac 1 {n^{O(\log k)}}}
.
\]  Given ${\ol{p}}$, the value of $\eps>0$, a confidence parameter $\delta>0$, and access to random samples from ${\cal S}_{\ol{p}} \ast f$, the algorithm runs in time ${\poly(n^{\log k},1/\eps,\log(1/\delta))}$ and with probability $1-\delta$  outputs a distribution $g: \mathbb{S}_n \to \R$ such that $\dtv(f,g) \leq \eps.$
\end{theorem}

Our second algorithmic result, which is similar in spirit to~\Cref{thm:ub-symmetric}, is for the heat kernel noise model: 

\begin{theorem} [Algorithm for heat kernel noise] \label{thm:ub-heat}
There is an algorithm with the following guarantee: Let $f$ be an unknown $\eps$-heavy distribution over $\mathbb{S}_n$ with $|\supp(f)| \leq k$.  Let $t \in \R^+$ be any value that is $O(n \log n)$.  Given $t$, the value of $\eps>0$, a confidence parameter $\delta>0$, and access to random samples from ${\cal H}_{t} \ast f$, the algorithm runs in time ${\poly(n^{\log k},1/\eps,\log(1/\delta))}$ and with probability $1-\delta$ outputs a distribution $g: \mathbb{S}_n \to \R$ such that $\dtv(f,g) \leq \eps.$
\end{theorem}

Recalling that the uniform random walk on the Cayley graph of $\mathbb{S}_n$ mixes in $\Theta(n \log n)$ steps, we see that the algorithm of~\Cref{thm:ub-heat} is able to handle quite high levels of noise and still run quite efficiently (in quasi-polynomial time).

\medskip

Our third positive result, for the Cayley-Mallows model, displays an intriguing qualitative difference from Theorems~\ref{thm:ub-symmetric} and~\ref{thm:ub-heat}.  To state our result, let us define the function $\mathsf{dist} : \R^+ \times \mathbb{N} \rightarrow \R^+$ as follows:
\[
\mathsf{dist}(\theta, \ell) := \min_{j \in \{1,\dots,\ell\}} \big| e^\theta - j \big|,
\]
so $\dist(\theta,\ell)$ measures the minimum distance between $e^\theta$ and any integer in $\{1,\dots,\ell\}$. \Cref{thm:ub-Mallows} gives an algorithm which can be quite efficient for the Cayley-Mallows noise model if the noise parameter $\theta$ is such that $\dist(\theta,\log k)$ is not too small:

\begin{theorem} [Algorithm for the Cayley-Mallows model]\label{thm:ub-Mallows}
There is an algorithm with the following guarantee: Let $f$ be an unknown $\eps$-heavy distribution over $\mathbb{S}_n$ with $|\supp(f)| \leq k$.
Given $\theta>0$, the value of $\eps > 0$, a confidence parameter $\delta > 0$, and access to random samples from ${\cal M}_\theta \ast f$, the algorithm runs in time $\poly(n^{\log k},1/\eps,\log(1/\delta),\dist(\theta,\log k)^{-\sqrt{\log k}})$ and with probability $1-\delta$ outputs a distribution $g: \mathbb{S}_n \to \R$ such that $\dtv(f,g) \leq \eps.$

\ignore{
}
\end{theorem}

As alluded to earlier, as $\theta$ approaches $0$ the difficulty of learning in the ${\cal M}_\theta$ noise model increases (and indeed learning becomes impossible at $\theta = 0$); since for small $\theta$ we have $\dist(\theta,\ell) \approx \theta$, this is accounted for by the $\dist(\theta,\log k)^{-\sqrt{\log k}}$ factor in our running time bound above.  However, for larger values of $\theta$ the $\dist(\theta,\log k)^{-\sqrt{\log k}}$ dependence may strike the reader as an unnatural artifact of our analysis:  is it really hard to learn when $\theta$ is very close to $\ln 2 \approx 0.63147$, easy when $\theta$ is very close to $\ln 2.5 \approx 0.91629,$ and hard again when $\theta$ is very close to $\ln 3 \approx 1.09861$? Perhaps surprisingly, the answer is yes:  it turns out that the $\dist(\cdot,\cdot)$ parameter captures a \emph{fundamental barrier} to learning in the Cayley-Mallows model.  We establish this by proving the following lower bound for the Cayley-Mallows model, which shows that a dependence on $\dist$ as in~\Cref{thm:ub-Mallows} is in fact inherent in the problem:

\begin{theorem}~\label{thm:lower-bound-intro}
Given $j \in \N$, there are infinitely many values of $k$ and $m =m(k)\ignore{ \approx {\frac {\log k}{\log \log k}}}$ such that the following holds:  Let $\theta>0$ be such that $|e^\theta-j|\leq \eta \leq 1/2$, and let $A$ be any algorithm which, when given access to random samples from ${\cal M}_\theta \ast f$ where $f$ is a distribution over $\mathbb{S}_{{m}}$ with $|\supp(f)| \leq k$,  with probability at least 0.51 outputs a distribution $h$ over $\mathbb{S}_{{m}}$ that has $\dtv(f,h) \leq 0.99$.  Then $A$ must use $\eta^{-\Omega\left(\sqrt{\frac {\log k}{\log \log k}}\right)}$ samples.
\end{theorem}

\subsection{Relation to prior work} \label{sec:priorwork}
Starting with the work of 
Mallows~\cite{mallows1957non}, there is a rich line of work in machine learning and statistics on probabilistic models of ranking data, see e.g.~\cite{marden2014analyzing, lebanon2002cranking, busse2007cluster, mandhani2009tractable, meilua2010dirichlet, lu2011learning}. In order to describe the prior works which are most relevant to our paper, it will be useful for us to define the \emph{Kendall-Mallows} model (referred to in the literature just as the Mallows model) in slightly more detail than we gave earlier. Introduced by Mallows~\cite{mallows1957non}, the  {Kendall-Mallows} model is quite similar to the Cayley-Mallows model that we consider --- it is specified by a parametric family of distributions $\{\calM_{\tau, \theta}\}_{\theta \in \mathbb{R}^+}$ and a central permutation $\sigma \in \mathbb{S}_n$, and a draw from the model is generated as follows: sample $\bpi \sim \calM_{\tau, \theta}$ and output $\bpi \cdot \sigma$. The distribution $\calM_{\tau, \theta}$ assigns probability weight $e^{-\theta K(\pi, e)}/Z_K(\theta)$ to the permutation $\pi$ where $Z_K(\theta) = \sum_{\pi \in \mathbb{S}_n} e^{-\theta K(\pi,e)}$ is the normalizing constant and $K(\cdot, \cdot)$ 
is the Kendall $\tau$-distance (defined next): 
\begin{definition}~\label{def:KT}
The \emph{Kendall $\tau$-distance} $K: \mathbb{S}_n \times \mathbb{S}_n \rightarrow \mathbb{R}^{\ge 0}$ is a distance metric on $\mathbb{S}_n$ defined as 
\[
K(\pi, \pi') = \{(i,j): i <j \textrm{ and } ((\pi(i) < \pi(j)) \oplus (\pi'(i) < \pi'(j)) =1) \}
\]
In other words, $K(\pi,\pi')$ is the number of inversions between $\pi$ and $\pi'$. Like the Cayley distance, the Kendall $\tau$-distance is also a right-invariant metric. Another equivalent way to define $K(\cdot, \cdot)$ is to consider the undirected graph on $\mathbb{S}_n$ where vertices $\pi_1$ $\pi_2$ share an edge if and only $\pi_1 = \tau \cdot \pi_2$ where $\tau$ is an \emph{adjacent transposition} -- in other words, $\tau = (i, i+1)$ for some $1 \le   i < n$. Then $K(\cdot, \cdot)$ is defined as the shortest path metric on this graph. From this perspective, the difference between the Kendall $\tau$-distance and the Cayley distance is that the former only allows adjacent transpositions while the latter allows all transpositions. 
\end{definition}

\medskip
\noindent {\bf Learning mixture models:} As mentioned earlier, probabilistic models of ranking data have been studied extensively in probability, statistics and machine learning. Models that have been considered in this context include the Kendall-Mallows model~\cite{mallows1957non, lu2011learning, meilua2007consensus, gladkich2018}, the Cayley-Mallows model (and generalizations of it)~\cite{fligner1986distance, murphy2003mixtures, mukherjee2016estimation, DH92, diaconis1988group, Ewens72}  and the heat kernel random walk model~\cite{kondor2002diffusion, KB10, jiao2018kendall}, among others. In contrast, within theoretical computer science interest in probabilistic models of ranking data is somewhat more recent, and the best-studied model in this community is the  Kendall-Mallows model. Braverman and Mossel~\cite{braverman2008noisy} initiated this study and (among other results)  gave an efficient algorithm to recover a single Kendall-Mallows model from random samples. The question of learning mixtures of $k$ Kendall-Mallows models was raised soon thereafter, and 
and Awasthi \emph{et~al.}~\cite{awasthi2014learning} gave an efficient algorithm for the case $k=2$. We note two key distinctions between our work and that of \cite{awasthi2014learning}: (i) our results apply to the Cayley-Mallows model rather than the Kendall-Mallows model, and (ii) the work of \cite{awasthi2014learning} allows for the two components in the mixture to have two different noise parameters $\theta_1$ and $\theta_2$ whereas our mixture models allow for only one noise parameter $\theta$ across all the components.  

Very recently, Liu and Moitra~\cite{liu2018} extended the result of  \cite{awasthi2014learning} to any constant $k$. In particular, the running time of the \cite{liu2018} algorithm scales as $n^{\poly(k)}$.  It is interesting to contrast our results with those of \cite{liu2018}. Besides the obvious difference in the models treated (namely Kendall-Mallows in \cite{liu2018} versus Cayley-Mallows in this paper), another significant difference is that our running time scales only quasipolynomially in $k$ versus exponentially in $k$ for \cite{liu2018}.  (In fact, \cite{liu2018} shows that an exponential dependence on $k$ is necessary for the problem they consider.) Another difference is that their algorithm allows each mixture component to have a different noise parameter $\theta_i$ whereas our result requires the same noise parameter $\theta$ across the mixture components. We observe that one curious feature of the algorithm of \cite{liu2018} is the following: When all the noise parameters $\{\theta_i\}_{1\le i \le k}$ are \emph{well-separated} (meaning that for all $i \not = j$, $|\theta_i - \theta_j| \ge \gamma$), then the running time of \cite{liu2018} can be improved to $\poly(n) \cdot 2^{\poly(k)}$. This suggests that the case when all $\theta_i$ are the same might be the hardest for the Liu-Moitra~\cite{liu2018} algorithm.  

Finally, we note that while the analysis in this paper does not immediately extend to the Kendall-Mallows model (see Section~\ref{sec:discussion} for more details), we point out that there is a sense in which the Kendall-Mallows and Cayley-Mallows models are fundamentally incomparable. This is because, while the results of \cite{liu2018} show that mixtures of Kendall-Mallows models are identifiable whenever each $\theta_i \not =1$, Theorem~\ref{thm:lower-bound-intro} shows that mixtures of Cayley-Mallows models are not identifiable at various larger values of $\theta$ such as $\ln 2, \ln 3, \dots$, even when all of the noise parameters are the same value $\theta$ which is provided to the algorithm.

\subsection{Our techniques} \label{sec:techniques}

\ignore{
\noindent {\bf Learning from noisy rankings.}}
A key notion for our algorithmic approach is that of the \emph{marginal} of a distribution $f$ over $\mathbb{S}_n$:

\begin{definition} \label{def:marginal}
Fix $f: \mathbb{S}_n \to [0,1]$ to be some distribution over $\mathbb{S}_n$. Let $t \in \{1,\dots,n\}$, let $\bar{i}=(i_1,\dots,i_t)$ be a vector of $t$ distinct elements of $\{1,\dots,n\}$ and likewise $\bar{j}=(j_1,\dots,j_t)$. We say the  \emph{$(\bar{i},\bar{j})$-marginal of $f$} is the probability
\[
\Prx_{\bsigma \sim f}[\bsigma(i_1)=j_1 \text{~and~} \cdots \text{~and~} \bsigma(i_t)=j_t]
\]
that for all $\ell=1,\dots,t$, the $i_\ell$-th element of a random $\bsigma$ drawn from $f$ is $j_\ell$.  When $\bar{i}$ and $\bar{j}$ are of length $t$ we refer to such a probability as a \emph{$t$-way marginal of $f$}.
\end{definition}

The first key ingredient of our approach for learning from noisy rankings is a reduction from the problem of learning $f$ (the unknown distribution supported on $k$ rankings $\sigma_1,\dots,\sigma_k$) given access to samples from ${\cal K} \ast f$, to the problem of estimating $t$-way marginals (for a not-too-large value of $t$).  More precisely, in~\Cref{sec:alg-recovery} we give an algorithm which, given the ability to efficiently estimate $t$-way marginals of $f$, efficiently computes a high-accuracy approximation for an unknown $\eps$-heavy distribution $f$ with support size at most $k$ (see~\Cref{thm:algo-rec}).  This algorithm builds on ideas in the population recovery literature, suitably extended to the domain $\mathbb{S}_n$ rather than $\{0,1\}^n$. 

With the above-described reduction in hand, in order to obtain a positive result for a specific noise model $\cal K$ the remaining task is to develop an algorithm $A_\marg$ which, given access to noisy samples from ${\cal K} \ast f$, can reliably estimate the required marginals.  In~\Cref{sec:marginal} we show that if the noise distribution ${\cal K}$ (a distribution over $\mathbb{S}_n$) is efficiently samplable, then given samples from ${\cal K} \ast f$, the time required to estimate the required marginals essentially depends on the minimum, over a certain set of matrices arising from the Fourier transform (over the symmetric group $\mathbb{S}_n$) of the noise distribution, of the minimum singular value of the matrix. (See~\Cref{thm:min-sing-val} for a detailed statement.)  At this point, we have reduced the algorithmic problem of obtaining a learning algorithm for a particular noise model to the analytic task of lower bounding the relevant singular values.  We carry out the required analyses on a noise-model-by-noise-model basis in Sections~\ref{sec:symmetric},~\ref{sec:heat}, and~\ref{sec:generalized}. These analyses employ ideas and results from the representation theory of the symmetric group and its connections to enumerative combinatorics; we give a brief overview of the necessary background in~\Cref{sec:rep}.

To establish our lower bound for the Cayley-Mallows model,~\Cref{thm:lower-bound-intro}, we exhibit two distributions $f_1$ and $f_2$ over the symmetric group such that the distributions of noisy rankings ${\cal M}_\theta \ast f_1$ and ${\cal M}_\theta \ast f_2$ have very small statistical distance from each other. Not surprisingly, the inspiration for this construction also comes from the representation theory of the symmetric group; more precisely, the two above-mentioned distributions are obtained from the character (over the symmetric group) corresponding to a particular carefully chosen partition of $[n]$.  A crucial ingredient in the proof is the fact that characters of the symmetric group are rational-valued functions, and hence any character can be split into a positive part and a negative part; details are given in~\Cref{sec:lowerbound}.

\subsection{Discussion and future work} \label{sec:discussion}
In this paper we have considered three particular noise models --- symmetric noise, heat kernel noise, and Cayley-Mallows noise --- and given efficient algorithms for these noise models.  Looking beyond these specific noise models, though, our approach provides a general framework for obtaining algorithms for learning mixtures of noisy rankings.  Indeed, for essentially any efficiently samplable noise distribution ${\cal K}$, given access to samples from ${\cal K} \ast f$ our approach reduces the algorithmic problem of learning $f$ to the analytic problem of lower bounding the minimum singular values of matrices arising from the Fourier transform of ${\cal K}$ (see~\Cref{thm:min-sing-val}).  We believe that this technique may be useful in a broader range of contexts, e.g.~to obtain results analogous to ours for the original Kendall-Mallows model or for other noise models.  

As is made clear in Sections~\ref{sec:symmetric},~\ref{sec:heat}, and~\ref{sec:generalized}, the representation-theoretic analysis that we require for our noise models is  facilitated by the fact that each of the noise distributions considered in those sections is a  \emph{class function} (in other words, the value of the distribution on a given input permutation depends only on the cycle structure of the permutation).  Extending the kinds of analyses that we  perform to other noise models which are not class functions is a technical challenge that we leave for future work.


\section{Algorithmic recovery of sparse functions} \label{sec:alg-recovery}

The main result of this section is the reduction alluded to in~\Cref{sec:techniques}.  In more detail, we give an algorithm which, given the ability to efficiently estimate $t$-way marginals, efficiently computes a high-accuracy approximation for an unknown $\eps$-heavy distribution $f$ with support size at most $k$:

\begin{theorem}~\label{thm:algo-rec}
Let $f$ be an unknown $\eps$-heavy distribution over $\mathbb{S}_n$ with $|\supp(f)| \leq k$.
Suppose there is an algorithm $A_\marg$ with the following property:  given as input a value $\delta>0$ and two vectors $\bar{i}=(i_1,\dots,i_t)$ and $\bar{j}=(j_1,\dots,j_t)$ each composed of $t$ distinct elements of $\{1,\dots,n\},$ algorithm $A_\marg$ runs in time $T(\delta,t,k,n)$ and outputs an additively $\pm \delta$-accurate estimate of the $(\bar{i},\bar{j})$-marginal of $f$ (recall~\Cref{def:marginal}). Then there is an algorithm $A_\learn$ with the following property:  given the value of $\epsilon$, algorithm $A_\learn$ runs in time $\poly(n/\eps,n^{\log k})\cdot T({\frac {\eps}{2k^{O(\log k)}}}, 2\log k, k^2, n)$\ignore{
\rnote{This was ``$T\big(\frac{\epsilon}{k^{4 \log k}}, 2 \log k,  2k^2, n \big)$'' but I'm not sure why}} and returns a function $g: \mathbb{S}_n \rightarrow \mathbb{R}^+$ such that $\Vert f- g \Vert_1 \le \epsilon$. 
\end{theorem}

Looking ahead, given~\Cref{thm:algo-rec}, in order to obtain a positive result for a specific noise model $\cal K$ the remaining task is to develop an algorithm $A_\marg$ which, given access to noisy samples from $ {\cal K} \ast f$, can reliably estimate the required marginals.  The algorithm is given in~\Cref{sec:marginal} and the detailed analyses establishing its efficiency for each of the noise models (by bounding minimum singular values of certain matrices arising from each specific noise distribution) is given in Sections~\ref{sec:symmetric},~\ref{sec:heat}, and~\ref{sec:generalized}.

\subsection{A useful structural result} \label{sec:structural}

The following structural result on functions from $\mathbb{S}_n$ to $\R$ with small support will be useful for us:

\begin{claim}
[Small-support functions are correlated with juntas]
\label{clm:WY-n-ary}
Fix $1 \leq \ell \leq n$ and let $g: [n]^\ell \rightarrow \mathbb{R}$ be such that $\Vert g \Vert_1=1$ and $|\supp(g)| \leq k$. There is a subset $U\subseteq [n]$ and a list of values $\alpha_1, \ldots, \alpha_{|U|} \in [n]$ such that $|U| \le \log k$ and 
\begin{equation} \label{eq:U}
\left|
\sum_{x \in [n]^\ell} g(x) \cdot \Ind[x_i = \alpha_i \text{~for all~}i \in U]
\right|
\ge k^{-O(\log k)}.
\end{equation}
\end{claim}

\Cref{clm:WY-n-ary} is reminiscent of analogous structural results for functions over $\{0,1\}^\ell$ which are implicit in the work of~\cite{WY12} (specifically, Theorem~1.5 of that work), and indeed~\Cref{clm:WY-n-ary} can be proved by following the techniques of~\cite{WY12}.  Michael Saks~\cite{Saks:18comm} has communicated to us an alternative, and arguably simpler, argument for the relevant structural result over $\zo^\ell$; here we follow that alternative argument (extending it in the essentially obvious way to the domain $[n]^\ell$ rather than $\{0,1\}^\ell$).

\begin{proof}
 Let the support of $g$ be $S \subseteq [n]^\ell$. Note that since $|S| \le k$, there must exist some set of $k' := \min\{k,\ell\}$ coordinates such that any two elements of $S$ differ in at least one of those coordinates. Without loss of generality, we assume that this set is the first $k'$ coordinates $\{1,\dots,k'\}.$

We  prove~\Cref{clm:WY-n-ary} by analyzing an iterative process that iterates over the coordinates $1,\dots,k'$.  At the beginning of the process, we initialize a set $\coord_{\mathsf{live}}$ of ``live coordinates'' to be $[k']$, initialize a set $\constr$ of constraints to be initially empty, and initialize a set  $S_{\mathsf{live}}$ of ``live support elements'' to be the entire support $S$ of $g$.
We will see that the iterative process maintains the following invariants:

\begin{itemize}
\item [(I1)] The coordinates in $\coord_{\mathsf{live}}$ are sufficient to distinguish between the elements in ${S}_{\mathsf{live}}$, i.e.~any two distinct strings in ${S}_{\mathsf{live}}$ have distinct projections onto the coordinates in $\coord_{\mathsf{live}}$;
\item [(I2)] The only elements of $S$ that satisfy all the constraints in $\constr$ are the elements of ${S}_{\mathsf{live}}$. 
\end{itemize} 

Before presenting the iterative process we need to define some pertinent quantities.
For each coordinate $j \in \coord_{\mathsf{live}}$ and each index $\alpha \in [n]$, we define
\[
\mathsf{Wt}(j, \alpha) := \sum_{x \in {S}_{\mathsf{live}}: x_j = \alpha} |g(x)|,
\]
the \emph{weight} under $g$ of the live support elements $x$ that have $x_j=\alpha$, and we define
\[
\mathsf{Num}(j, \alpha) := |\{x \in {S}_{\mathsf{live}} \, : \, x_j = \alpha\}|,
\]
the number of live support elements $x$ that have $x_j=\alpha$ (note that $\mathsf{Num}(j,\alpha)$ has nothing to do with $g$).
It will also be useful to have notation for fractional versions of each of these quantities, so we define
\[
\mathsf{FracWt}(j,\alpha) := {\frac {\mathsf{Wt}(j,\alpha)} {\sum_{x \in {S}_{\mathsf{live}}} |g(x)|}}.
\quad \quad \quad \text{and} \quad \quad \quad
\mathsf{Frac}(j,\alpha) := {\frac {\mathsf{Num}(j,\alpha)} {|{S}_{\mathsf{live}}|}}
\]
Note that for any $j \in \coord_{\mathsf{live}}$ we have that $\sum_\alpha \mathsf{Num}(j,\alpha) = |{S}_{\mathsf{live}}|$, or equivalently $\sum_\alpha \mathsf{Frac}(j,\alpha) = 1.$

For each coordinate $j \in \coord_{\mathsf{live}}$, we write $\mathsf{MAJ}(j)$ to denote the element $\beta \in [n]$ which is such that $\mathsf{Num}(j,\beta) \ge \mathsf{Num}(j,\alpha)$ for all $\alpha \in [n]$ (we break ties arbitrarily).   Finally, we let $\mathsf{FracWtMaj}(j) = \mathsf{FracWt}(j,\mathsf{MAJ}(j))$.

Now we are ready to present the iterative process:
 
\begin{enumerate}
\item If every $j \in \coord_{\mathsf{live}}$ has $\mathsf{FracWtMaj}(j) > 1- \frac{1}{10k'}$\footnote{Note that this means almost all of the weight under $g$ of the live support elements is on elements that all agree with the majority value on coordinate $j$. Note further that if $\coord_{\mathsf{live}}$ is empty then this condition trivially holds.}, then halt the process. Otherwise, let $j$ be any element of $\coord{\mathsf{live}}$ for which $\mathsf{FracWtMaj}(j) \le 1- \frac{1}{10k'}$. 

\item For this coordinate $j$, choose $\alpha \in [n]$ which maximizes the ratio $\frac{\mathsf{FracWt}(j, \alpha)}{\mathsf{Frac}(j,\alpha)}$ (or equivalently, maximizes $\frac{\mathsf{FracWt}(j, \alpha)}{\mathsf{Num}(j,\alpha)}$) subject to $\mathsf{Frac}(j,\alpha) \not = 0$ and $\alpha \not = \mathsf{MAJ}(j)$. 

\item Add the constraint $x_j = \alpha$ to $\constr$, remove $j$ from $\coord_{\mathsf{live}}$, and remove all $x$ such that $x_j \not = \alpha$ from ${S}_{\mathsf{live}}$. Go to Step~1.   

\end{enumerate} 

When the iterative process ends, suppose that the set $\constr$ is $\{x_{j_1} = \alpha_1, \ldots, x_{j_\ell} = \alpha_\ell\}$. Then we claim that~\Cref{eq:U} holds for $U = \{j_1, \ldots, j_\ell\}$.

To argue this, we first observe that both invariants (I1) and (I2) are clearly maintained by each round of the iterative process. We next observe that each time a pair $(j,\alpha)$ is processed in Step~3, it holds that $\mathsf{Frac}(j,\alpha) \le \frac{1}{2}$, and hence each round shrinks ${S}_{\mathsf{live}}$ by a factor of at least $2$. Thus, after $\log k$ steps, the set ${S}_{\mathsf{live}}$ must be of size at most $1$ and hence the process must halt.  (Note that the claimed bound $|U| \leq \log k$ follows from the fact that the process runs for at most $\log k$ stages.)

Next, note that when the process halts, by a union bound over the at most $k'$ coordinates in $\coord_{\mathsf{live}}$ it holds that  
\[
\sum_{x \in {S}_{\mathsf{live}}: x_j = \mathsf{MAJ}(j) \text{~for all~}j \in \coord_{\mathsf{live}}} |g(x)| \ge \frac{9}{10} \cdot \sum_{x \in {S}_{\mathsf{live}}} |g(x)|. 
\]
On the other hand, by the first invariant (I1),
the cardinality of the set $\{x \in {S}_{\mathsf{live}}: x_j = \mathsf{MAJ}(j)$ for all $j \in \coord_{\mathsf{live}}\}$ is precisely $1$. This immediately implies that almost all of the weight of $g$, across elements of ${S}_{\mathsf{live}}$, is on a single element; more precisely, that
\[
\left|\sum_{x \in {S}_{\mathsf{live}}} g(x)\right| \ge \frac{4}{5} \cdot \sum_{x \in {S}_{\mathsf{live}}} |g(x)| ,
\]
from which it follows that 
\begin{equation}~\label{eq:lowerbound-1}
\left|
\sum_{x \in [n]^\ell} g(x) \cdot \Ind[x_i = \alpha_i \text{~for all~}i \in U]
\right| \ge \frac{4}{5} \cdot \sum_{x \in {S}_{\mathsf{live}}} |g(x)| . 
\end{equation}

So to establish~\Cref{eq:U}, it remains only to establish a lower bound on $\sum_{x \in {S}_{\mathsf{live}}} |g(x)|$ when the process terminates. To do this, let us suppose that the process runs for $T$ steps where in the $t^{th}$ step the coordinate chosen is  $j_t$. Now, at any stage $t$, we have
$$
\frac{\sum_{\beta \in \coord_{\mathsf{live}}: \beta \not = \mathsf{MAJ}(j_t)} \mathsf{FracWt}(j_t,\beta)}{\sum_{\beta \in \coord_{\mathsf{live}}: \beta \not = \mathsf{MAJ}(j_t)} \mathsf{Frac}(j_t,\beta)}
 \ge \frac{1}{10k'}.$$
(because the denominator is at most $1$ and since the process does not terminate, the numerator is at least $\frac{1}{10k}$). As a result, we get that if the constraint chosen at time $t$ is $x_{j_t} = \alpha_t$, then 
\begin{equation}~\label{eq:proportion}
\frac{\mathsf{FracWt}(j_t,\alpha_t)}{\mathsf{Frac}(j_t,\alpha_t)} \ge \frac{1}{10k'}.
\end{equation}
By~\Cref{eq:proportion}, when the process halts we have
 $$
 \sum_{x \in {S}_{\mathsf{live}}} |g(x)| =  \prod_{t =1}^T \mathsf{FracWt}(j_t, \alpha_t) \ge \frac{1}{(10k')^T} \prod_{t =1}^T \mathsf{Frac}(j_t, \alpha_t).
  $$
But since at least one element remains, we have that $\prod_{t =1}^T \mathsf{Frac}(j_t, \alpha_t) \ge \frac{1}{k}$, and since $T \le \log k$, we conclude (recalling that $k' \leq k$) that 
  $$
  \sum_{x \in {S}_{\mathsf{live}}} |g(x)| \ge k^{-O(\log k)}. 
  $$
  Combining with (\ref{eq:lowerbound-1}), this yields the claim. 
\end{proof}

\subsection{Proof of~\Cref{thm:algo-rec}}

The idea of the proof is quite similar to the algorithmic component of several recent works on population recovery~\cite{moitra2013polynomial, WY12, lovett2015improved, de2016noisy}. Given any function $f: \mathbb{S}_n \rightarrow \mathbb{R}$ and any integer $i \in \{1,\dots,n\}$, we define the function $f_{i}:[n]^i \rightarrow \mathbb{R}$ as follows: 
\begin{equation}~\label{eq:def-f-ell}
f_i(x_1, \ldots, x_i) := \sum_{\sigma \in \mathbb{S}_n} f(\sigma) \cdot \Ind[\sigma(1) = x_1 \wedge \ldots \wedge \sigma(i) =x_i]. 
\end{equation}

At a high level, the algorithm~$A_\learn$ of~\Cref{thm:algo-rec} works in stages, by successively reconstructing $f_0, \ldots, f_n$.  In each stage it uses the procedure described in the following claim, which says that high-accuracy approximations of the $(\log k)$-marginals \emph{together with the support of $f_\ell$} (or a not-too-large superset of it) suffices to reconstruct $f_\ell$:
\begin{claim}~\label{clm:known-support}
Let $f_\ell$ be an unknown  distribution over $[n]^\ell$ supported on a given set $S$ of size $k$. 
There is an algorithm $A_{\mathrm{one-stage}}$ which has the following guarantee:  The algorithm is given as input $\delta>0$,  and parameters $\beta_{J,y}$ (for every set $J \subseteq [\ell]$ of size at most $\log k$ and every $y \in [n]^J$) which satisfy
\[
\left|\beta_{J,y} - \sum_{x \in S} f(x) \cdot \Ind[x_i = y_i \text{~for all~}i\in J] \right| \le \delta.
\] 
$A_{\mathrm{one-stage}}$ runs in time $\mathsf{poly}(n, \ell^{\log k})$
and outputs a function $\tilde{f}: [n]^\ell \to [0,1]$ such that
$\|f-\tilde{f}\|_1 \leq \delta \cdot k^{O(\log k)}.$
\end{claim}

\begin{proof}
We consider a linear program which has a variable $s_x$ for each $x \in S$ (representing the probability that $f$ puts on $x$) and is defined by the following constraints: 

\begin{enumerate}
\item $s_x \ge 0$ and $\sum_{x \in S} s_x=1$. 
\item For each $J \subseteq [\ell]$ of size at most $\log k$ and each $y \in [n]^J$, include the constraint 
\begin{equation} \label{eq:constraint}
\left| \beta_{J,y} - \sum_{x \in S} s_x \cdot  \Ind[x_i = y_i \text{~for all~}i\in J] \right| \le \delta.
\end{equation}
\end{enumerate}
Algorithm~$A_{\mathsf{one-stage}}$ sets up and solves the above linear program (this can clearly be done in time $\mathsf{poly}(n, \ell^{\log k})$).  We observe that the linear program is feasible since by definition $s_x=f_\ell(x)$ is a feasible solution. To prove the claim it suffices to show that every feasible solution is $\ell_1$-close to $f_\ell$; so let $f^\ast(x)$ denote any other feasible solution to the linear program, and let $\eta$ denote $\|f^\ast-f_\ell\|_1.$ Define $h(x) =  f^{\ast}(x) - f_\ell(x),$ so $\|h\|_1=\eta.$ By~\Cref{clm:WY-n-ary}, we have that there is a subset $J \subseteq [\ell]$ of size at most $\log k$ and a $y \in [n]^\ell$ such that
\begin{equation} \label{eq:goodie}
\left|\sum_x h(x) \cdot  \Ind[x_i = y_i \text{~for all~}i\in J]
\right| 
\geq \eta \cdot k^{-O(\log k)}. 
\end{equation} \label{eq:bag}
On the other hand, since both $f_\ell(x)$ and $f^\ast(x)$ are feasible solutions to the linear program, by the triangle inequality it must be the case that
\begin{equation} 
\left|\sum_x h(x) \cdot  \Ind[x_i = y_i \text{~for all~}i\in J]
\right|  \leq 2 \delta. 
\end{equation}
Equations~\ref{eq:goodie} and~\ref{eq:bag} together give the desired upper bound on $\eta$, and the claim is proved.
\end{proof}

Essentially the only remaining ingredient required to prove~\Cref{thm:algo-rec} is a procedure to find (a not-too-large superset of) the support of $f$.  This is given by the following claim, which inductively uses the algorithm $A_{\mathrm{one-stage}}$ to successively construct suitable (approximations of) the support sets for $f_1,\dots,f_n.$

\begin{claim} \label{claim:a}
Under the assumptions of~\Cref{thm:algo-rec}, there is an algorithm $A_{\mathrm{support}}$ with the following property: given as input a value $\delta>0$, algorithm $A_{\mathrm{support}}$ runs in time $\poly(n/\eps,n^{\log k}) \cdot T({\frac {\eps}{2k^{O(\log k)}}}, 2\log k, k^2, n)$ and for each $\ell=1,\dots,n$ outputs a set $S'_{(\ell)}$ of size at most $k$ which contains the support of $f_\ell$.
\end{claim}

\begin{proof}
The algorithm $A_{\mathrm{support}}$ works inductively, where at the start of stage $\ell$ (in which it will construct the set $S'_{(\ell)}$) it is assumed to have a set $S'_{(\ell-1)}$ with $|S'_{(\ell-1)}| \leq {k}$ which contains the support of $f_{\ell-1}.$  (Note that at the start of the first stage $\ell=1$ this holds trivially since $f_0$ trivially has empty support).

Let us describe the execution of the $\ell$-th stage of $A_{\mathrm{support}}$.  For $1 \le \ell \le n$, we define the set $S_{\textsf{marg},\ell}$ as follows: 
\[
S_{\textsf{marg},\ell} = \big\{ t: \sum_{\sigma \in \mathbb{S}_n} f(\sigma) \cdot \Ind[\sigma(\ell) = t] >0\big\}.
\]
Observe that in time $\mathsf{poly}(n/\epsilon) \cdot T(\frac{\epsilon}{4}, 1, k , n)$,  we can compute $f(\sigma) \cdot \Ind[\sigma(\ell) = t]$ up to error $\pm \epsilon/4$ (denote this estimate by $\beta_{\ell,t}$) for all $1 \le t \le n$. Since $f$ is $\epsilon$-heavy, we have that $$ 
t \in S_{\textsf{marg},\ell} \textrm{ implies } \beta_{\ell,t} \ge  \frac{3\epsilon}{4} \quad \textrm{and} \quad  t \not \in S_{\textsf{marg},\ell} \textrm{ implies } \beta_{\ell,t} \le  \frac{\epsilon}{4}. 
$$ 
Consequently, we can compute the set $S_{\textsf{marg},\ell} $ 
in time  $\mathsf{poly}(n/\epsilon) \cdot T(\frac{\epsilon}{4}, 1, k , n)$. The final observation is that 
the set $S^\ast_{(\ell)}$ (of cardinality at most $k^2$) obtained by appending each final $\ell$-th character from $S_{\textsf{marg},\ell}$ to each element of $S'_{(\ell-1)}$  must contain the support $S_{(\ell)}$ of $f_\ell$.  Set $\delta = \frac{\epsilon}{2 k^{O(\log k)}}$; by the assumption of~\Cref{thm:algo-rec}, in time $T({\frac {\eps}{2k^{O(\log k)}}}, 2\log k, k^2, n)$\ignore{\rnote{This was ``$T\big(\frac{\epsilon}{(nk)^{\log k + \log n}}, \log k + \log n, n \cdot k, n \big)$'' but I'm not sure why}} it is possible to obtain additively $\pm \delta$-accurate estimates of each of the $(2 \log k)$-way marginals of $f_\ell$.  In the $\ell$-th stage, algorithm $A_{\mathrm{support}}$ runs $A_{\mathrm{one-stage}}$ using $S^\ast_{(\ell)}$ and these estimates of the marginals; by~\Cref{clm:known-support}, this takes time $\poly(n/\eps,n^{\log k})$ and yields a function $\tilde{f}_\ell: [n]^\ell \to [0,1]$ such that $\|f_\ell-\tilde{f}_\ell\|_1 \leq {\frac {\delta} {2 k^{O(\log k)}}} \cdot k^{O(\log k)} = \eps/4.$ Since by assumption $f$ is $\eps$-heavy, it follows that any element $x$ in the support of $\tilde{f}_\ell$ such that $\tilde{f}_\ell(x)\le\eps/4$ must not be in the support of $f_\ell$; so the algorithm removes all such elements $x$ from $S^\ast_{(\ell)}$ to obtain the set $S'_{(\ell)}.$  This resulting $S'_{(\ell)}$ is precisely the support of $f_\ell$, and is clearly of size at most $k.$
\end{proof}

Finally, the overall algorithm $A_{\mathrm{learn}}$ works by running $A_{\mathrm{support}}$ 
to get the set $S'=S'_{(n)}$ of size at most $k$ which is the support of $f_n=f$, and then uses $S'$ and the algorithm $A_{\mathrm{marginal}}$ from the assumptions of~\Cref{thm:algo-rec}) to run algorithm $A_{\mathrm{one-stage}}$ and obtain the required $\eps$-accurate approximator $g$ of $f$. This concludes the proof of~\Cref{thm:algo-rec}.

\section{Computing limited way marginals from noisy samples}~\label{sec:marginal}

 Recall that the noisy ranking learning problems we consider are of the following sort: There is a known noise distribution ${\cal K}$ supported on $\mathbb{S}_n$, and an unknown $k$-sparse $\eps$-heavy distribution $f: \mathbb{S}_n \rightarrow [0,1]$. Each sample provided to the learning algorithm is generated by  the following probabilistic process: independent draws of $\bpi \sim \mathcal{K}$ and $\bsigma \sim f$ are obtained, and the sample given to the learner is $(\bpi \bsigma) \in \mathbb{S}_n$.  By the reduction established in~\Cref{thm:algo-rec}, in order to give an algorithm that learns the distribution $f$ in the presence of a particular kind of noise ${\cal K}$, it suffices to give an algorithm that can efficiently estimate $t$-way marginals given samples $\bpi \bsigma \sim  {\cal K} \ast f.$

The main result of this section,~\Cref{thm:min-sing-val}, gives such an algorithm.  Before stating the theorem we need some terminology and notation and we need to recall some necessary background from representation theory of the symmetric group (see~\Cref{sec:rep} for a detailed overview of all of the required background).

First, let ${\cal K}$ be a distribution over $\mathbb{S}_n$ (which should be thought of as a noise distribution as described earlier). We say that ${\cal K}$ is \emph{efficiently samplable} if there is a $\poly(n)$-time randomized algorithm which takes no input and, each time it is invoked, returns an independent draw of $\bpi \sim {\cal K}.$

Next, we recall that a \emph{partition} $\lambda$ of the natural number $n$ (written ``$\lambda \vdash n$'')
is a vector of natural numbers $(\lambda_1,\dots,\lambda_k)$ where $\lambda_1 \ge \lambda_2 \ge \ldots \ge  \lambda_k>0$ and $\lambda_1 + \ldots + \lambda_k=n$ (see~\Cref{sec:repsym} for more detail).  For two partitions $\lambda$ and $\mu$ of $n$, we say that \emph{$\mu$ dominates $\lambda$}, written $\mu \rhd \lambda$, if $\sum_{j \le i} \mu_j \ge \sum_{j \le i} \lambda_j$ for all $i >0$ (see~\Cref{def:dominance-order}).  Given any $\lambda \vdash n$, let $\mathsf{Up}(\lambda)$ denote the set of all partitions $\mu \vdash n$ such that $\mu \rhd \lambda.$

We recall that a \emph{representation} of the symmetric group $\mathbb{S}_n$ is a group homomorphism from $\mathbb{S}_n$ to $\mathbb{C}^{m \times m}$ (see~\Cref{sec:rep}).  We further recall that for each partition $\lambda \vdash n$ there is a corresponding \emph{irreducible} representation, denoted $\rho_\lambda$ (see~\Cref{sec:repsym}).  For a matrix $M$ we write $\sigma_{\min}(M)$ to denote the smallest singular value of $M$.  Given a partition $\lambda \vdash n$ we define the value $\sigma_{\min, \mathsf{Up}(\lambda),{\cal K}}$ to be
\begin{equation} \label{eq:sigma-min}
\sigma_{\min, \mathsf{Up}(\lambda),{\cal K}} := \min_{\mu \in \mathsf{Up}(\lambda)} \sigma_{\min} (\widehat{\mathcal{K}}(\rho_{\mu})), 
\end{equation}
the smallest singular value across all Fourier coefficients of the noise distribution of irreducible representations corresponding to partitions that dominate $\lambda$.  (We recall that the Fourier coefficients of functions over the symmetric group, and indeed over any finite group, are matrices; see~\Cref{sec:repsym}.)

Finally, for $0 \leq \ell \leq n-1$ we define the partition $\lambda_{\mathsf{hook},\ell} \vdash n$ to be 
\[
\lambda_{\mathsf{hook},\ell} := (n-\ell,1,\ldots, 1).
\]

Now we can state the main result of this section:
\begin{theorem} \label{thm:min-sing-val}
Let ${\cal K}$ be an efficiently samplable distribution over $\mathbb{S}_n$. Let $f$ be an unknown distribution over $\mathbb{S}_n$. There is an algorithm $A_{\mathrm{marginal}}$ with the following properties:  $A_{\mathrm{marginal}}$ receives as input a parameter $\delta>0$, a confidence parameter $\tau > 0$, a pair of ${\ell}$-tuples $\bar{i}=(i_1,\dots,i_\ell) \in [n]^\ell$, $\bar{j}=(j_1,\dots,j_\ell) \in [n]^\ell$ each composed of $\ell$ distinct elements, and has access to random samples from ${\cal K} \ast f$. Algorithm $A_{\mathrm{marginal}}$ runs in time $\mathsf{poly}(\binom{n}{\ell}, \delta^{-1}, \sigma^{-1}_{\min, \mathsf{Up}(\lambda_{\mathsf{hook},\ell}),{\cal K},}\log(1/\tau))$ and outputs a value $\kappa_{\ol{i},\ol{j}}$ which with probability at least $1-\tau$ is a $\pm \delta$-accurate estimate of the $(\ol{i},\ol{j})$-marginal of $f$.\ignore{ 
$$
\left| \kappa_{\ol{i},\ol{j}} - \sum_{\sigma \in \mathbb{S}_n} f(\sigma) \cdot \Ind[\sigma(i_1)=j_1 \text{~and~}
\sigma(i_\ell)=j_\ell] \right| \le \delta. 
$$}
\end{theorem}

We will use the following claim to prove~\Cref{thm:min-sing-val}:

\begin{claim}~\label{clm:representation}
Let $\rho: \mathbb{S}_n \rightarrow \mathbb{C}^{m \times m}$ be any unitary representation of $\mathbb{S}_n$, let $\mathcal{K}$ be any efficiently samplable distribution over $\mathbb{S}_n$, and let $\sigma_{\min}$ denote the smallest singular value of $\widehat{\mathcal{K}}(\rho)$. Let $f$ be an unknown distribution over $\mathbb{S}_n$.  There is an algorithm which, given random samples from ${\cal K} \ast f$  and an error parameter $0<\delta<1$, runs in time $\mathsf{poly}(m,n, \sigma_{\min}^{-1},\delta^{-1})$ and with high probability\ignore{\rnote{Do we want to get into the whole error probability analysis or just leave it at this?}} outputs a matrix $M_{f,\rho}$ such that $\Vert M_{f,\rho} - \widehat{f}(\rho) \Vert \le \delta$. 
\end{claim}

\begin{proof}
Let $\eta_1,\eta_2>0$ denote two error parameters that will be fixed later.
Since $f$ is a distribution, the Fourier coefficient $\hat{f}(\rho)$ is equal to $\E_{\bsigma \sim f}[\rho(\bsigma)]$.
Consequently, since ${\cal K}$ is assumed to be efficiently samplable and the algorithm is given samples from ${\cal K} \ast f$, by sampling from ${\cal K}$ and from $ {\cal K} \ast f$ it is straightforward to obtain matrices
$M_1,M_2$ in time $\poly(m,n,\log(1/\tau))$ which with probability $1-\tau$ satisfy
\[
\Vert M_1- \widehat{\mathcal{K}}(\rho) \Vert_2 \le \eta_1 \ \textrm{and} \ \Vert M_2- \widehat{ \mathcal{K}\ast f}(\rho) \Vert_2 \le \eta_2. 
\]
Now we recall the following matrix perturbation inequality (see 
Theorem~2.2 of \cite{stewart1977}):

\begin{lemma}~\label{lem:matrix-inv}
Let $A \in \mathbb{R}^{n \times n}$ be a non-singular matrix and further let
$\Delta A \in \mathbb{R}^{n \times n}$ be such that $\Vert \Delta A\Vert_2 \cdot \Vert A^{-1} \Vert_2<1$. Then $A+\Delta A$ is non-singular. Further, if $\gamma=1-\Vert A^{-1} \Vert_2 \Vert \Delta A\Vert_2 $, then 
\[
\Vert A^{-1} - (A+\Delta A)^{-1} \Vert_2 \leq  \frac{\Vert A^{-1} \Vert_2^2 \Vert \Delta A \Vert_2 }{\gamma}. 
\]
\end{lemma}
Let us now set the error parameters $\eta_1$ and $\eta_2$ as follows (recall that $\delta < 1$):
\begin{equation}~\label{eq:eta-error-bounds}
\eta_1 = \min\big\{\frac{\delta \cdot \sigma_{\min}^2}{4}
,  
\frac{\delta \cdot \sigma_{\min}}{4}\big\} \ \  \textrm{and} \ \ \eta_2 = \min\{\frac{\delta \cdot \sigma_{\min}}{4}, 1\}. 
\end{equation}
Applying Lemma~\ref{lem:matrix-inv} with $\widehat{{\cal K}}(\rho)$ in place of $A$ and $M_1 - \widehat{{\cal K}}(\rho)$ in place of $\Delta A$, using (\ref{eq:eta-error-bounds}) 
(more precisely, the upper bound $\eta_1 \leq \delta \cdot \sigma_{\min}^2/4$ in 
the numerator and the upper bound $\eta_1 \leq \delta \cdot \sigma_{\min}/4$ in the denominator) 
we get that 
\begin{equation}~\label{eq:inv-diff}
\Vert M_1^{-1} - \widehat{\mathcal{K}}(\rho)^{-1} \Vert_2 \leq \frac{\Vert \widehat{\mathcal{K}}(\rho)^{-1} \Vert_2^2 \cdot \Vert M_1 - \widehat{\mathcal{K}} (\rho)\Vert_2}{1-\Vert \widehat{\mathcal{K}}(\rho)^{-1} \Vert_2 \cdot \Vert M_1 - \widehat{\mathcal{K}} (\rho)\Vert_2} \le
\frac{\delta}{3}. 
\end{equation}
Now using $\widehat{ \mathcal{K} \ast f} (\rho) = \widehat{\mathcal{K}}(\rho) \cdot \widehat{f}(\rho)$, we get 

\begin{eqnarray}
\Vert M_1^{-1} \cdot M_2 - \widehat{f}(\rho) \Vert_2  &=& 
\Vert M_1^{-1} \cdot M_2 - \widehat{\mathcal{K}}(\rho)^{-1} \cdot \widehat{ \mathcal{K} \ast f}(\rho) \cdot   \Vert_2 \nonumber\\
&\le&  \Vert  M_1^{-1} \cdot M_2 - M_1^{-1} \cdot \widehat{ \mathcal{K} \ast f}(\rho) \Vert_2  + \Vert M_1^{-1} \cdot \widehat{ \mathcal{K} \ast f}(\rho) - \widehat{\mathcal{K}}(\rho)^{-1} \cdot \widehat{ \mathcal{K} \ast f}(\rho)\Vert_2 \nonumber
\\
&\le&  \Vert M_1^{-1} \Vert_2 \cdot \Vert M_2- \widehat{ \mathcal{K} \ast f}(\rho) \Vert_2   +   \Vert M_1^{-1} -  \widehat{\mathcal{K}}(\rho)^{-1} \Vert_2  \cdot \Vert \widehat{ \mathcal{K} \ast f}(\rho) \Vert_2
\nonumber\\
&\le&  \Vert M_1^{-1} \Vert_2 \cdot \eta_2 +  \Vert \widehat{ \mathcal{K} \ast f}(\rho) \Vert_2 \cdot \frac{\delta}{3}. \ \  \quad (\textrm{using (\ref{eq:inv-diff})}) \nonumber\\
&\le& \eta_2 \left(\|\widehat{{\cal K}}(\rho)^{-1}\|_2 + 
\|M^{-1} - \widehat{{\cal K}}(\rho)^{-1}\|_2 
\right) +  \Vert \widehat{ \mathcal{K} \ast f}(\rho) \Vert_2 \cdot \frac{\delta}{3}. \ \  \quad (\textrm{using (\ref{eq:inv-diff})}) \nonumber\\
&\le& \sigma_{\min}^{-1} \cdot \eta_2 + {\frac \delta 3} \cdot \eta_2 + 
\Vert \widehat{ \mathcal{K} \ast f}(\rho) \Vert_2 \cdot \frac{\delta}{3}. \label{eq:error-bound-3}
\end{eqnarray}

Next we use the following fact, which is an easy consequence of the triangle inequality and the assumption that $\rho$ is unitary: 
\begin{fact}~\label{fact:unitary}
Let $\rho: \mathbb{S}_n \rightarrow \mathbb{C}^{m \times m}$ be a unitary representation and let $g: \mathbb{S}_n \rightarrow \mathbb{R}^+$. Then we have that $\Vert \widehat{g}(\rho) \Vert_2 \le \Vert g \Vert_1$.
\end{fact}
Combining this fact with (\ref{eq:error-bound-3}) and (\ref{eq:eta-error-bounds}), since $\Vert {\cal K} \ast f \Vert_1 = 1$, we get that 
\[
\Vert M_1^{-1} \cdot M_2 - \widehat{f}(\rho) \Vert_2  \leq \sigma_{\min}^{-1} \cdot \eta_2 +  {\frac \delta 3} \cdot \eta_2 +  \frac{\delta}{3} \le \frac{\delta}{4} + {\frac \delta 3} + \frac{\delta}{3} <\delta. 
\]
This concludes the proof of ~\Cref{clm:representation}.
\end{proof}

With~\Cref{clm:representation} in hand we are ready to prove~\Cref{thm:min-sing-val}:

\medskip
\noindent {\emph Proof of~\Cref{thm:min-sing-val}.}
Let $\tau_{\lambda_{\mathsf{hook},\ell}}$ be the permutation representation corresponding to the partition $\lambda_{\mathsf{hook},\ell}$; for conciseness we subsequently write $\rho$ for $\tau_{\lambda_{\mathsf{hook},\ell}}$.  \Cref{def:replambda} immediately gives that the dimension of $\rho$ is $\binom{n}{\ell}$. Observe that $\rho$ is a unitary representation.
Let $\sigma_{\min}$ denote the smallest singular value of $\widehat{\mathcal{K}}(\rho)$;  applying~\Cref{clm:representation}, we get an algorithm running in time $\mathsf{poly}(\binom{n}{\ell}, \sigma_{\min}^{-1}, \delta)$ which outputs a matrix $M_{f,\rho}$ such that $\Vert M_{f,\rho} - \widehat{f}(\rho) \Vert \le \delta$. Next, we observe that the Young tableaux corresponding to the partition $\lambda_{\mathsf{hook},\ell}$ (which, recalling~\Cref{def:replambda}, index the rows and columns of $\rho(\cdot)$) correspond precisely to ordered $t$-tuples of distinct entries of $[n]$. If $\mathsf{Y}_{\lambda_{\mathsf{hook},\ell}, i} = \ol{i}$ and $\mathsf{Y}_{\lambda_{\mathsf{hook},\ell}, j} = \ol{j}$, then it follows that
\[
\widehat{f}(\rho) (i,j) = \sum_{\sigma \in \mathbb{S}_n} f(\sigma) \cdot \Ind[f(i_1) = j_1 \text{~and~} \cdots
\text{~and~} 
f(i_\ell) = j_\ell)],
\]
which is the $(\ol{i},\ol{j})$-marginal of $f$ as desired; so the output of the algorithm is $M_{f,\rho}(\ol{i},\ol{j})$.

To finish the correctness argument it remains only to argue that $\sigma^{-1}_{\min}$ is at most $\poly(\sigma^{-1}_{\min, \mathsf{Up}(\lambda_{\mathsf{hook},\ell})}).$ To see that this is indeed, the case, we observe that by~\Cref{thm:James-submodule}, the permutation representation $\tau_{\lambda_{\mathsf{hook},\ell}}$ block diagonalizes into a direct sum of irreducible representations $\rho_{\mu}$ where each $\mu$ belongs to $\mathsf{Up}(\lambda_{\mathsf{hook},\ell})$. This finishes the proof of~\Cref{thm:min-sing-val}. \qed

\subsection{Efficient samplability of our noise distributions} \label{sec:samplability}

In order to apply~\Cref{thm:min-sing-val} to a particular noise distribution ${\cal K}$ we need to confirm that ${\cal K}$ is efficiently samplable; we now do this for each of the three noise models that we consider.  It is immediate from the definition that it is straightforward (given $\ol{p}$ ) to efficiently generate a random $\bsigma$ drawn from the symmetric noise distribution ${\cal S}_{\ol{p}}$, and the same is true for the heat kernel noise distribution ${\cal H}_t$.

For the generalized Mallows model ${\cal M}_\theta$, the characterization $\Pr_{\bsigma \sim {\cal M}_{\theta}}[\bsigma = \pi] = e^{-\theta d(\pi,e)}/Z(\theta)$ given earlier does not directly yield an efficient sampling algorithm, since it may be hard to compute or approximate the normalizing factor $Z(\theta) = \sum_{\pi \in \mathbb{S}_n} e^{-\theta d(\pi,e)}.$  Instead, we recall (see e.g.~Section~2.1 of \cite{DSC98}) that the Metropolis algorithm can be used to efficiently perform a random walk on $\mathbb{S}_n$ whose unique stationary distribution is the generalized Mallows distribution ${\cal M}_{\theta}$.  (Each step of the random walk can be carried out efficiently because it is computationally easy to compute the Cayley distance between two permutations:  if $\pi$ is the permutation that brings $\sigma$ to $\tau$, then the Cayley distance $d(\sigma,\tau)$ is $n-\cycles(\pi)$ where $\cycles(\pi)$ is the number of cycles in $\pi$.)  It is known (see e.g. Theorem~2 of \cite{DH92}) that this random walk has rapid convergence, and consequently it is indeed possible to sample efficiently from ${\cal M}_\theta$ (up to an exponentially small statistical distance which can be ignored in our applications since our algorithms use a sub-exponential number of samples).


\section{Representations of symmetric noise}~\label{sec:symmetric}

In this section we establish lower bounds on the smallest singular value for the relevant matrices corresponding to ``symmetric noise'' ${\cal S}_{\ol{p}}$ on $\mathbb{S}_n$. In more detail,  the main result of this section is the following lower bound:

\begin{lemma} \label{lem:lower-bound-symmetric}
Let $\ell \in \{1,\dots,n\}$ and let $\ol{p}=(p_0,\dots,p_n) \in \Delta^n$ (i.e.~$\ol{p}$ is a non-negative vector whose entries sum to 1) which is such that
\[
\sum_{j=0}^{n-\ell} p_j \geq \kappa.
\]  Then (recalling~\Cref{eq:sigma-min}) we have that
\begin{equation} \label{eq:sigma-min-symmetric-rep1}
\sigma_{\min,\mathsf{Up}(\lambda_{\mathsf{hook},\ell}),{\cal S}_{\ol{p}}} \geq {\frac {\kappa} {n^\ell}}.
\end{equation}
\end{lemma}

\subsection{Setup}
To analyze the smallest singular value of $\widehat{{\cal S}_{\ol{p}}}(\rho_\mu)$ (as required by the definition of $\sigma_{\min,\mathsf{Up}(\lambda_{\mathsf{hook},\ell}),{\cal S}_{\ol{p}}}$), we start by observing that symmetric noise is a \emph{class function} (meaning that it is invariant under conjugation, see~\Cref{def:class-function}):

\begin{claim}~\label{clm:inv-conj-symmetric}
For any vector $\ol{p}=({p}_0, \ldots, {p}_n) \in \Delta^n$, the distribution ${\cal S}_{\ol{p}}$ (viewed as a function from $\mathbb{S}_n$ to $[0,1]$) is a class function ( i.e.~${\cal S}_{\ol{p}}(\pi)=
{\cal S}_{\ol{p}}(\tau \pi \tau^{-1})$ for every $\pi,\tau \in \mathbb{S}_n$).
\end{claim}

\begin{proof}

For $0 \le j \le n$, let $\ol{e}_j$ denote the vector in $\R^{n+1}$ which has a 1 in the $j$-th position and a 0 in every other position. 
By linearity, to prove~\Cref{clm:inv-conj-symmetric} it suffices to prove that ${\cal S}_{\ol{e}_j}$ is invariant under conjugation for every $j$; to establish this, it suffices to show that ${\cal S}_{\ol{e}_j}$ is invariant under conjugation by any transposition $\tau$. By symmetry, it suffices to consider the transposition $\tau=(1,2)$.

We observe that ${\cal S}_{\ol{e}_j}$ is a uniform average of $\mathbb{U}_A$ over all ${n \choose j}$ subsets $A$ of $[n]$ of size exactly $j$.  Now we consider two cases:  the first is that $|A \cap \{1,2\}|$ is 0 or 2. In this case it is easy to see that $\mathbb{U}_A$ does not change under conjugation by the transposition $(1,2)$. The remaining case is that $|A \cap \{1,2\}|=1$; in this case it is easy to see that conjugation by $(1,2)$ converts $\mathbb{U}_A$ into $\mathbb{U}_{A \Delta \{1,2\}}$.  Since the collection of size-$j$ sets $A$ with $A \cap \{1,2\}=\{1\}$ are in 1-1 correspondence with the collection of size-$j$ sets $A$ with $A \cap \{1,2\}=\{2\}$, it follows that ${\cal S}_{\ol{e}_j}$ is invariant under conjugation by $\tau=(1,2)$, and the proof is complete.
\ignore{
\gray{
Consequently, we just need to prove that $\tau \cdot \mathcal{S}_{\mathbf{j}} \cdot \tau^{-1} = \mathcal{S}_{\mathbf{j}}$. Because of symmetry, we can just assume that $\tau = (1,2)$. 

We now make the following two observations. 
\begin{enumerate}
\item If $|\mathcal{A} \cap \{1,2\}| = 0$ or $2$, then 
$\tau \cdot \mathbb{U}_{\mathcal{A}} \cdot \tau^{-1} = \mathbb{U}_{\mathcal{A}}$ -- in other words, $\mathbb{U}_{\mathcal{A}}$ does not change under conjugation by $\{1,2\}$. 
\item If $|\mathcal{A} \cap \{1,2\}|= 1$, then $\tau \cdot \mathbb{U}_{\mathcal{A}} \cdot \tau^{-1} = \mathbb{U}_{\mathcal{A} \oplus \{1,2\}}$. 
\end{enumerate}
Combining the two items along with the fact that the operator $\oplus \{1,2\}$ is an involution on sets, we obtain that $\mathcal{S}_{\mathbf{j}}$ is invariant under conjugation by $(1,2)$. This finishes the proof. 
}
}
\end{proof}

Before stating the next lemma we remind the reader that for partitions $\mu \vdash m, \lambda \vdash n$ where $m\leq n$, we write $\paths(\mu,\lambda)$ to denote the number of paths from $\mu$ to $\lambda$ in Young's lattice (see~\Cref{sec:repsym} and~\Cref{thm:branch}). We write $\mathsf{Triv}_j$ to denote the trivial partition $(j)$ of $j$.

\begin{lemma}~\label{lem:eigen-sym}
Let $\lambda \vdash n$ and let $\rho_{\lambda}$ be the corresponding irreducible representation of $\mathbb{S}_n$. Given $\ol{p}=(p_0,\dots,p_n) \in \Delta^n$, we have that
\begin{equation} \label{eq:cplambda}
\widehat{\mathcal{S}_{\ol{p}}}(\rho_{\lambda}) =\mathsf{c}(\ol{p},\lambda) \cdot \mathsf{Id} \textrm{ where } \ \ \mathsf{c}(\ol{p},\lambda):=\frac{ \sum_{j=0}^n p_j \cdot \paths(\mathsf{Triv}_j, \lambda)}{\mathsf{dim}(\rho_{\lambda})}. 
\end{equation}
\end{lemma}

\begin{proof}
By~\Cref{clm:inv-conj-symmetric}, we have that ${\cal S}_{\ol{p}}$ is a class function, so we may apply~\Cref{lem:class-func} to conclude that
\[
\widehat{\mathcal{S}_{\ol{p}}}(\rho_{\lambda}) = \mathsf{c}(\ol{p},\lambda) \cdot \mathsf{Id},
\]
where 
\[
\mathsf{c}(\ol{p},\lambda) = \frac{1}{\mathsf{dim}(\rho_{\lambda})} \cdot \bigg( \sum_{\sigma \in \mathbb{S}_n} \mathcal{S}_{\ol{p}}(\sigma) \cdot \chi_{{\lambda}} (\sigma) \bigg)
\]
and $\chi_\lambda$ denotes the character of the irreducible representation $\rho_\lambda$. Thus it remains to show that $ \sum_{\sigma \in \mathbb{S}_n} \mathcal{S}_{\ol{p}}(\sigma) \cdot \chi_{{\lambda}} (\sigma) $ is equal to the numerator of~\Cref{eq:cplambda}. By definition of ${\cal S}_{\ol{p}}$, we have that
\begin{eqnarray}~\label{eq:split-1}
\sum_{\sigma \in \mathbb{S}_n} \mathcal{S}_{\ol{p}}(\sigma) \cdot \chi_{{\lambda}} (\sigma)  &=& \sum_{0 \le j \le n} \ol{p}_j \mathop{\mathbf{E}}_{{\bf {\cal A}} : |{\bf {\cal A}}|=j} \mathop{\mathbf{E}}_{\bsigma \in \mathbb{U}_{\mathcal{A}}} \chi_{{\lambda}}(\bsigma).
\end{eqnarray}
We proceed to analyze $\mathop{\mathbf{E}}_{\bsigma \in \mathbb{U}_{\mathcal{A}}} \chi_{{\lambda}}(\bsigma)   $.
 Let $\rho_{\lambda}^{\mathcal{A}}$ denote the representation $\rho_{\lambda}$ restricted to the subgroup $\mathbb{S}_{\mathcal{A}}$. 
By Theorem~\ref{thm:branch}, the representation $\rho_{\lambda}^{\mathcal{A}}$ splits as follows: 
\[
\rho_{\lambda}^{\mathcal{A}} = \mathop{\oplus}_{\mu \vdash |\mathcal{A}| } \paths(\mu, \lambda) \rho_{\mu}.
\]
Thus, we have that 
\[
\mathop{\mathbf{E}}_{\bsigma \in \mathbb{U}_{\mathcal{A}}} \chi_{{\lambda}}(\bsigma) = \sum_{\mu \vdash |\mathcal{A}|} \paths(\mu, \lambda) \mathop{\mathbf{E}}_{\bsigma \in \mathbb{U}_{\mathcal{A}}}\chi_{\mu}(\bsigma) =\paths (\mathsf{Triv}_{|\mathcal{A}|}, \lambda).  
\]
The second equality follows from that fact that if $\mu$ is a non-trivial partition of $|{\cal A}|$ then $\mathop{\mathbf{E}}_{\bsigma \in \mathbb{U}_{\mathcal{A}}}\chi_{\mu}(\bsigma) =0$, while  if $\mu = \mathsf{Triv}_{|\mathcal{A}|}$ then $\mathop{\mathbf{E}}_{\bsigma \in \mathbb{U}_{\mathcal{A}}}\chi_{\mu}(\bsigma) =1$. 
Plugging this into (\ref{eq:split-1}) we get that
$ \sum_{\sigma \in \mathbb{S}_n} \mathcal{S}_{\ol{p}}(\sigma) \cdot \chi_{{\lambda}} (\sigma) = \sum_{j=0}^n p_j \cdot \paths(\mathsf{Triv}_j, \lambda)$, and the lemma is proved. 
\end{proof}

\subsection{Proof of~\Cref{lem:lower-bound-symmetric}} \label{sec:pf-lower-bound-sym}

We recall from~\Cref{eq:sigma-min} that
\[
\sigma_{\min, \mathsf{Up}(\lambda_{\mathsf{hook},\ell}),{\cal S}_{\ol{p}}} := \min_{\mu \in \mathsf{Up}(\lambda_{\mathsf{hook},\ell})} \sigma_{\min} (\widehat{{\cal S}_{\ol{p}}}(\rho_{\mu})).
\]
Fix any $\mu \in \mathsf{Up}(\lambda_{\mathsf{hook},\ell})$, so $\mu$ is a partition of $n$ of the form $(n-\ell',\ell_2,\dots,\ell_r)$ where $\ell' \leq \ell.$
By~\Cref{lem:eigen-sym} we have that the smallest singular value of $\widehat{{\cal S}_{\ol{p}}}(\rho_{\mu})$ is
\begin{equation} \label{eq:cpmu}
\mathsf{c}(\ol{p},\mu):=\frac{ \sum_{j=0}^n p_j \cdot \paths(\mathsf{Triv}_j, \mu)}{\mathsf{dim}(\rho_{\mu})}.
\end{equation}
To upper bound $\mathsf{dim}(\rho_\mu)$, we observe that
\[
\dim(\rho_\mu) \leq \dim(\tau_\mu) = 
{n \choose n-\ell',\ell_2,\dots,\ell_r}
\leq {\frac {n!}{(n-\ell')!}} \leq n^{\ell'} \leq n^\ell,
\]
where the first inequality is by~\Cref{thm:James-submodule}. For the numerator, we observe that if $j \leq n-\ell$ then there is at least one path in the Young lattice from $\mathsf{Triv}_j$ to $\mu$, so under the assumptions of~\Cref{lem:lower-bound-symmetric} the numerator of~\Cref{eq:cpmu} is at least $\kappa.$ This proves the lemma.
\qed


\section{Representations of heat kernel noise}~\label{sec:heat}

In this section, analogous to~\Cref{sec:symmetric}, we  lower bound~\Cref{eq:sigma-min} when the noise distribution ${\cal K}$ is ${\cal H}_t$, corresponding to ``heat kernel noise'' at temperature parameter $t$:

\begin{lemma} \label{lem:lower-bound-heat}
Let $t\geq 1$ and let $\ell \in \{1,\dots,cn\}$ for some suitably small universal constant $c>0$.  Then we have that
\begin{equation} \label{eq:sigma-min-symmetric-rep-2}
\sigma_{\min,\mathsf{Up}(\lambda_{\mathsf{hook},\ell}),{\cal H}_{t}} \geq {\frac 1 2} \cdot e^{-O(\ell t)/n}.
\end{equation}
\end{lemma}

\subsection{Setup}
Let $\trans: \mathbb{S}_n \to [0,1]$ be the following probability distribution over $\mathbb{S}_n$:
\[
\trans(\pi) = 
\begin{cases}
1/n   & \text{if $\pi$ is the identity,}\\
2/n^2 & \text{if $\pi$ is a transposition,}\\
0     & \text{otherwise.} 
\end{cases}
\]

Since $\trans(\pi)$ depends only on the cycle structure of $\pi$, the function $\trans(\cdot)$ is a class function.
Fix any $\mu \in \mathsf{Up}(\lambda_{\mathsf{hook},\ell})$, so $\mu$ is a partition of $n$ of the form $(\mu_1,\dots,\mu_r)$ where $\mu_1 \geq n-\ell.$
As in the proof of~\Cref{lem:eigen-sym}
 we may apply~\Cref{lem:class-func} to conclude that
\[
\widehat{\trans}(\rho_{\mu}) = \mathsf{c}_{\trans,\mu} \cdot \mathsf{Id}
\]
for some constant $\mathsf{c}_{\trans,\mu}$.
By Corollary~1 of Diaconis and Shahshahani \cite{DS81}, we have that
\begin{equation} \label{eq:DS}
\mathsf{c}_{\trans,\mu} = 
{\frac 1 n} + {\frac {n-1} n} \cdot {\frac {\chi_{\mu}(\tau)}{\dim(\rho_\mu)}},
\end{equation}
where as before $\chi_\mu$ denotes the character of the irreducible representation $\rho_\mu$ and $\tau$ is any transposition. \cite{DS81} further shows that for $\rho_\mu$ an irreducible representation of $\mathbb{S}_n$ with $\mu$ as above and $\tau$ any transposition, it holds that
\begin{equation} 
{\frac {\chi_{\mu}(\tau)}{\dim(\rho_\mu)}} = {\frac 1 {n(n-1)}} \cdot \sum_{j=1}^{r} (\mu_j - j)(\mu_j - j + 1) - j(j-1). \label{eq:DS2} 
\end{equation}
In our setting we have
\begin{equation}
(\ref{eq:DS2}) \geq
{\frac {(n-\ell)(n-\ell-1)}{n(n-1)}} + {\frac 1 {n(n-1)}} \sum_{j=2}^r (\mu_j-j)(\mu_j-j+1) - j(j-1). \label{eq:pickle}
\end{equation}
where the inequality holds because $\mu_1 \geq n-\ell.$  Now, we observe that for each summand in~\Cref{eq:pickle}, we have
\begin{align*}
(\mu_j-j)(\mu_j-j+1) - j(j-1) &= \mu_j^2 - \mu_j(2j-1)\\
& \geq -\mu_j(2j-1)\\
&\geq {\frac {-\ell}{j-1}} \cdot (2j-1) \geq - 3\ell.
\end{align*}
The second inequality above holds because $\mu_2 + \cdots + \mu_j \leq \ell$ and the $\mu_j$'s are non-increasing, so $\mu_j \leq {\frac {\ell}{j-1}}.$ Since $r-1\leq \ell$, this means that
\[
(\ref{eq:DS2}) \geq
{\frac {(n-\ell)(n-\ell-1)}{n(n-1)}} - {\frac {3\ell^2}{n(n-1)}} \geq 1 - {\frac {O(\ell)} n},
\]
and recalling~\Cref{eq:DS} we get that
\begin{equation} \label{eq:ctranslb}
1 \geq \mathsf{c}_{\trans,\mu} \geq 1 - {\frac {O(\ell)} n}.
\end{equation}

\subsection{Proof of~\Cref{lem:lower-bound-heat}}

As in~\Cref{sec:symmetric}  we recall from~\Cref{eq:sigma-min} that
\[
\sigma_{\min, \mathsf{Up}(\lambda_{\mathsf{hook},\ell}),{\cal H}_{t}} := \min_{\mu \in \mathsf{Up}(\lambda_{\mathsf{hook},\ell})} \sigma_{\min} (\widehat{{\cal H}_{t}}(\rho_{\mu})),
\]
Fix any $\mu \in \mathsf{Up}(\lambda_{\mathsf{hook},\ell})$ (so $\mu$ is a partition of $n$ of the form $(\mu_1,\dots,\mu_r)$ where $\mu_1 \geq n-\ell$). 
We recall that the function ${\cal H}_t: \mathbb{S}_n \to [0,1]$ is defined by
\[
{\cal H}_t = \sum_{j=0}^\infty
\Pr_{\bT \sim \mathsf{Poi}(t)}[\bT=j](\trans)^j,
\]
where ``$(\trans)^T$'' denotes $T$-fold convolution of $\trans$. Since convolution corresponds to multiplication of Fourier coefficients, this gives that
\begin{equation} \label{eq:heat}
\widehat{{\cal H}_t}(\rho_\mu)= \mathsf{c}(t,\mu) \cdot \mathsf{Id}, \text{~where~}
\mathsf{c}(t,\mu) := 
\sum_{j=0}^\infty \Pr_{\bT \sim \mathsf{Poi}(t)}[\bT=j] (\mathsf{c}_{\trans,\mu})^j.
\end{equation}
Recalling \cite{Choi94} that the median of the Poisson distribution $\mathsf{Poi}(t)$ is at most $t+1/3$, we get that 
\[
\mathsf{c}(t,\mu) \geq {\frac 1 2} \cdot (\mathsf{c}_{\trans,\mu})^{t+1/3} \geq {\frac 1 2} \cdot e^{-O(\ell t)/n},
\]
(where the second inequality uses $\ell \leq cn$ and $t \geq 1$),
and the lemma is proved. \qed

\section{Representations of Cayley-Mallows model noise}~\label{sec:generalized}

In this section we  lower bound~\Cref{eq:sigma-min} when the noise distribution ${\cal K}$ is ${\cal M}_\theta$, corresponding to the Cayley-Mallows noise model with parameter $\theta$:

\begin{lemma} \label{lem:lower-bound-mallows}
Let $\theta>0$, let $\ell \in \{1,\dots,n\}$, and let $\eta := \dist(\theta,\ell) = \min_{j \in \{1,\dots,\ell\}} \big|e^\theta - j \big|.$  Then (recalling~\Cref{eq:sigma-min}) we have that
\begin{equation} \label{eq:sigma-min-symmetric}
\sigma_{\min,\mathsf{Up}(\mu_{\mathsf{hook},\ell}),{\cal M}_{\theta}} \geq (2n)^{-\ell} \eta^{2 \sqrt{\ell}}.
\end{equation}
\end{lemma}

Similar to the previous two sections,~\Cref{lem:lower-bound-mallows} follows immediately from the following lower bound on singular values of certain irreducible representations:
\begin{lemma}~\label{lem:eigen-mallows}
Let $\mu$ be a partition of $n$ of the form $(\mu_1,\dots,\mu_r)$ where $\mu_1 \geq n-\ell$.  Let $\theta >0$ and let $\eta := \dist(\theta,\ell) = \min_{j \in \{1,\dots,\ell\}} \big|e^\theta - j \big|.$  Then we have that 
\[
\widehat{\mathcal{M}_{\theta}}(\rho_{\mu}) = c_{\mu, \theta} \cdot \mathsf{Id} \ \textrm{ where } \ \ |c_{\mu, \theta}|  \ge (2n)^{-\ell} \eta^{2 \sqrt{\ell}}.
\]
\end{lemma}

To prove~\Cref{lem:eigen-mallows}, we will need the notions of \emph{content} and \emph{hook length} for boxes in a Young diagram:

\begin{definition}~\label{def:content} 
Let $\mu$ be a partition $\mu \vdash n$.  The \emph{hook length} of a box $u$ in the Young diagram for $\mu,$ denoted by $h(u)$, is the sum 
\[
\text{(\# of boxes to the right of $u$ in its row)} + 
\text{(\# of boxes below $u$ in its column)} + 1 \text{~(for $u$ itself)}.
\]
The \emph{content} $c(u)$ of a box $u$ is $c(u) := j-i,$ where $j$ is its column number (from the left, starting with column 1) and $i$ is its row number (from the top, starting with row 1).  
\end{definition}

The left portion of~\Cref{fig:yd-c} depicts a Young diagram annotated with the hook lengths of each of its boxes.
The right portion of~\Cref{fig:yd-c} depicts the same Young diagram annotated with the contents of each of its boxes.

\begin{center}
\begin{figure}[ht]
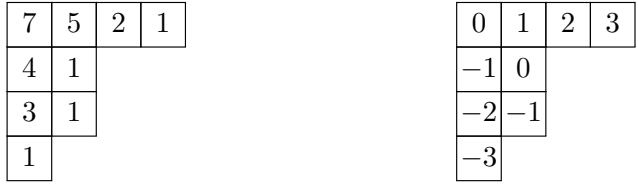

\quad
\quad
\quad
\quad
\quad
\quad
\quad
\quad
\quad
\quad
\begin{ytableau}
7 & 5 & 2  & 1 \\
4 & 1 \\
3 & 1 \\
1
\end{ytableau}
\quad
\quad
\quad 
\quad
\quad
\quad
\quad 
\quad
\quad
\begin{ytableau}
0 & 1 & 2  & 3 \\
-1 &0 \\
-2 & -1 \\
-3
\end{ytableau}
\caption{On the left is a Young diagram in which each box has been labeled with its hook length; on the right is a Young diagram in which each box has been labeled with its content. }
\label{fig:yd-c}
\end{figure}
\end{center}

We will need the following technical result to prove Lemma~\ref{lem:eigen-mallows}:

\begin{lemma}~\label{lem:formula}
Let $\mu \vdash n$ and let $\chi_{\mu}$ be the corresponding character in $\mathbb{S}_n$. For any $q \in \mathbb{R}$, 
\[
\frac{1}{n!} \sum_{\sigma \in \mathbb{S}_n} \chi_{\mu}(\sigma) \cdot q^{\cycles(\sigma)} =\prod_{u \in \mu} \frac{q+ c(u)}{h(u)}, 
\]
where the subscript ``$u \in \mu$'' means that $u$ ranges over all the boxes in the Young diagram corresponding to $\mu$. 
\end{lemma}
\begin{proof}
The above identity is given as Exercise~7.50 in Stanley's book~\cite{StanleyECv2}. For the sake of completeness, we provide the proof here. 

For any $\bar{t}=(t_1, \ldots, t_n)$, we define the  polynomial 
\[
a_{\bar{t}}(x_1, \ldots, x_n) := \mathsf{det} \begin{bmatrix}
    x_1^{t_1}       & x_2^{t_1} & x_3^{t_1} & \dots & x_{n}^{t_1} \\
    x_1^{t_2}       & x_2^{t_2} & x_3^{t_2} & \dots & x_{n}^{t_2} \\
    \hdotsfor{5} \\
    x_1^{t_n}       & x_2^{t_n} & x_3^{t_n} & \dots & x_{n}^{t_n}
\end{bmatrix}.
\]
Given any partition $\mu \vdash n$, we now define the Schur polynomial $s_{\mu}(x_1, \ldots, x_n)$ as
follows: Define $\bar{t}_{\mu} = (\mu_1 +n-1, \ldots, \mu_n +0)$ and $\bar{t}_{0} = (n-1, \ldots, 0)$. Then, 
\[
s_{\mu}(x_1, \ldots, x_n) := \frac{a_{\bar{t}_{\mu}}(x_1, \ldots,x_n)}
{a_{\bar{t}_{0}}(x_1, \ldots,x_n)}. 
\]
The denominator is just the Vandermonde determinant of the variables $(x_1, \ldots, x_n)$. As the polynomial $a_{\bar{t}_{\mu}}(x_1, \ldots,x_n)$ is alternating, it follows that $s_{\mu}(x_1, \ldots, x_n)$ is a polynomial (as opposed to a rational function) and further, it is symmetric. 

The following is a fundamental fact connecting Schur polynomials and cycles: For any $0 \le k \le n$, 
\begin{equation}~\label{eq:stanley-1}
s_{\mu}(\underbrace{1, \ldots, 1}_{k}, \underbrace{0,\ldots, 0}_{n-k}) = \sum_{\sigma \in \mathbb{S}_n} \frac{1}{n!} \cdot \chi_{\mu}(\sigma) \cdot k^{\cycles(\sigma)}
\end{equation}
(see equation~7.78 in \cite{StanleyECv2}).
On the other hand, there are known explicit formulas for evaluations of the Schur polynomial at specific inputs. In particular, Corollary~7.21.4 of \cite{StanleyECv2} states that 
\begin{equation}~\label{eq:stanley-2}
s_{\mu}(\underbrace{1, \ldots, 1}_{k}, \underbrace{0,\ldots, 0}_{n-k}) = \prod_{u \in \mu} \frac{k+ c(u)}{h(u)}. 
\end{equation}
Combining (\ref{eq:stanley-1}) and (\ref{eq:stanley-2}), we get that for any $0 \le k \le n$, 
we have 
\[
\frac{1}{n!} \sum_{\sigma \in \mathbb{S}_n} \chi_{\mu}(\sigma) \cdot k^{\cycles(\sigma)} =\prod_{u \in \mu} \frac{k+ c(u)}{h(u)}.
\]
However, note that both the left and the right hand sides can be seen as polynomials of degree at most $n$ in the variable $k$. Since they agree at $n+1$ values $k=0, \ldots, n$,  they must be identical as formal functions. This concludes the proof. 
\end{proof}

\begin{proofof}{Lemma~\ref{lem:eigen-mallows}}
Recall that the distribution ${\cal M}_\theta$ over $\mathbb{S}_n$ is defined by ${\cal M}_{\theta}(\pi) = e^{-\theta d(\pi,e)}/Z(\theta)$, where $Z(\theta) = \sum_{\pi \in \mathbb{S}_n} e^{-\theta d(\pi,e)}$ is a normalizing constant.  Since the Cayley distance $d(\sigma,\tau)$ is equal to $n-\cycles(\sigma^{-1}\tau)$, where $\cycles(\pi)$ is the number of cycles in $\pi$, we have that 
\[
{\cal M}_{\theta}(\pi)={\frac {e^{\theta \cdot \cycles(\pi)}} C},
\text{~where~}
C=\sum_{\pi \in \mathbb{S}_n} e^{\theta \cdot \cycles(\pi)}.
\]

Since the $\cycles(\cdot)$ function is a class function so is ${\cal M}_{\theta}$, so we can apply Lemma~\ref{lem:class-func} and we get that $\widehat{\mathcal{M}_{\theta}}(\rho_{\mu}) = c_{\mu, \theta} \cdot \mathsf{Id}$, where 
\[
c_{\mu, \theta} = \frac{\sum_{\sigma \in \mathbb{S}_n} \mathcal{M}_{\theta}(\sigma) \cdot \chi_{\mu}(\sigma)}{\mathsf{dim}(\rho_{\mu})} 
= \frac{\sum_{\sigma \in \mathbb{S}_n} e^{\theta \cdot \cycles(\sigma)} \cdot \chi_{\mu}(\sigma)}{\mathsf{dim}(\rho_{\mu}) \cdot (\sum_{\sigma \in \mathbb{S}_n} e^{\theta \cdot \cycles(\sigma)})}
=\frac{\sum_{\sigma \in \mathbb{S}_n} q^{\cycles(\sigma)} \cdot \chi_{\mu}(\sigma)}{\mathsf{dim}(\rho_{\mu}) \cdot (\sum_{\sigma \in \mathbb{S}_n} q^{\cycles(\sigma)})},
\text{~where~}q := e^{\theta}.
\]
We re-express the numerator by applying Lemma~\ref{lem:formula} to get 
\begin{equation}~\label{eq:nume}
\sum_{\sigma \in \mathbb{S}_n} q^{\cycles(\sigma)} \cdot \chi_{\mu}(\sigma)  = n! \cdot \prod_{u \in \mu} \frac{q+ c(u)}{h(u)}. 
\end{equation} 
To analyze the denominator of $c_{\mu,\theta}$, applying Lemma~\ref{lem:formula} to the trivial partition $\mathsf{Triv}_n = (n)$ of $n$ (the character of which is identically 1), we get that 
\begin{equation}~\label{eq:denom}
\sum_{\sigma \in \mathbb{S}_n} q^{\cycles(\sigma)}  = n! \cdot \prod_{u \in \mathsf{Triv}_n} \frac{q+ c(u)}{h(u)}
= q(q+1)\cdots(q+n-1).
\end{equation}
For the rest of the denominator, we recall the following well-known fact about the dimension of irreducible representations of the symmetric group:
\begin{fact} [Hook length formula, see e.g.~Theorem~3.41 of \cite{meliot2017representation}] \label{fact:dimension}
For $\mu \vdash n$, $\mathsf{dim}(\rho_{\mu}) = \frac{n!}{\prod_{u \in \mu} h(u)}$. 
\end{fact}
Combining (\ref{eq:nume}), (\ref{eq:denom}) and Fact~\ref{fact:dimension}, we get 
\begin{equation}~\label{eq:rexp}
c_{\mu, \theta} = \frac{\prod_{u \in \mu} (q+c(u))}{q(q+1)\cdots(q+n-1)}.
\end{equation}
Let $\mathcal{A}$ denote the set consisting of the cells of the Young diagram of $\mu$ which are not in the first row. Since $n-\mu_1=\ell'$ for some $\ell' \leq \ell$, the above expression simplifies to 
\begin{equation} \label{eq:simplified}
c_{\mu,  \theta}  = \frac{\prod_{u \in \mathcal{A}} (q+c(u))}{(q+n-\ell')\cdots(q+n-1)}. 
\end{equation}
To bound this ratio, first observe that both the numerator and denominator are $\ell'$-way products. There are two possibilities now: 

\begin{enumerate}
\item {\bf Case 1: $q \ge \ell+1$.} In this case we observe that each cell $u \in \mathcal{A}$ satisfies $c(u) \ge -\ell' \geq -\ell.$  Thus $c_{\mu,\theta}$ can be expressed as a product of $\ell'$ many fractions, each of which is at least $ \frac{q-\ell}{q+n-1} \ge \frac{1}{\ell+n}$. This implies that 
\[
c_{\mu, \theta} \geq \left(\frac{1}{n+\ell}\right)^{\ell'} \ge (2n)^{-\ell}. 
\]
\item 
{\bf Case 2: $q \leq\ell$.} In this case, the denominator of~\Cref{eq:simplified} is at most $(2n)^\ell$. To lower bound the numerator, observe that for every cell $u$ of ${\cal A}$, the value of $c(u)$ is an integer in $\{-\ell,\dots,\ell\}.$  Let $j_0$ and $j_1$ denote the two values in $\{-\ell,\dots,\ell\}$ for which $|q-j|$ achieves its smallest value $\eta$ and its next smallest value (note that these values are equal if $\eta=1/2$).  Next, we observe that at most $\sqrt{\ell}$ many cells of ${\cal A}$ have content equal to any given fixed integer value. Since $j_0$ and $j_1$ are the only possible values of $j \in \{-\ell,\dots,\ell\}$ for which $|q+j|<1$, it follows that
$$
\prod_{u \in \mathcal{A}} |(q+c(u))| \geq
\left( \prod_{u \in \mathcal{A}:c(u)=j_0} |(q+c(u))|\right) \cdot
\left( \prod_{u \in \mathcal{A}:c(u)=j_1} |(q+c(u))|\right) \ge \eta^{2 \sqrt{\ell}}.
$$
This finishes the proof. 

\end{enumerate}

\end{proofof}


\section{Our positive results for noisy rankings:  Putting the pieces together}

In this brief section we put all the pieces together to obtain our main positive results, Theorems~\ref{thm:ub-symmetric},~\ref{thm:ub-heat} and~\ref{thm:ub-Mallows}, for the symmetric, heat kernel, and generalized Mallows noise models respectively.

\medskip

\noindent {\bf Symmetric noise.} Under the assumptions of~\Cref{thm:ub-symmetric} (that $\sum_{j=0}^{n - \log k} p_j \geq {\frac 1 {n^{O(\log k)}}}$), taking $\ell=\log k$ in~\Cref{lem:lower-bound-symmetric}, we have that
$\sigma_{\min,\mathsf{Up}(\lambda_{\mathsf{hook},\log k}),{\cal S}_{\ol{p}}} \geq {\frac {1} {n^{O(\log k)}}}.$
Since (as discussed in~\Cref{sec:samplability}) ${\cal S}_{\ol{p}}$ is efficiently samplable given $\ol{p}$, by~\Cref{thm:min-sing-val} in time $\poly(n^{\log k},1/\delta,\log(1/\tau))$ with probability $1-\tau$ it is possible to obtain $\pm \delta$-accurate estimates of all of the $(\log k)$-way marginals of $f$.
Setting $\delta = {\frac {\eps} {2k^{O(\log k)}}}$ and applying~\Cref{thm:algo-rec}, we get~\Cref{thm:ub-symmetric}.

\medskip

\noindent {\bf Heat kernel noise.}  First observe that we may assume that the temperature parameter $t$ is at least 1 (since otherwise it is easy to artificially add noise to achieve $t = 1$).  Under the assumptions of~\Cref{thm:ub-heat} (that $t=O(n \log n)$), taking $\ell=\log k$ in~\Cref{lem:lower-bound-heat}, we have that
$\sigma_{\min,\mathsf{Up}(\lambda_{\mathsf{hook},\log k}),{\cal H}_{t}} \geq {\frac {1} {n^{O(\log k)}}}.$ \Cref{thm:ub-heat} follows as in the previous paragraph (this time using the efficient samplability of ${\cal H}_t$ given $t$).

\medskip

\noindent {\bf Cayley-Mallows noise.} Under the assumptions of~\Cref{thm:ub-Mallows}, taking $\ell=\log k$ in~\Cref{lem:lower-bound-mallows} we get that $\sigma_{\min,\mathsf{Up}(\lambda_{\mathsf{hook},\log k}),{\cal M}_{\theta}} \geq {\frac {1} {n^{O(\log k)}}} \cdot \dist(\theta, \log k)^{2\sqrt{\log k}}.$  \Cref{thm:ub-Mallows} follows as in the previous paragraph (this time using the efficient samplability of ${\cal M}_\theta$ given $\theta$).


\section{Lower bound for Cayley-Mallows models}~\label{sec:lowerbound}

Recall that because of the $\poly(\dist(\theta,\log k)^{-\sqrt{\log k}})$ dependence in~\Cref{thm:ub-Mallows}, the algorithm of that theorem is inefficient if $e^\theta$ is very close to an integer.
In this section we prove~\Cref{thm:lower-bound-intro}, which establishes that \emph{any algorithm} for learning in the presence of Cayley-Mallows noise \emph{must} be inefficient if $e^\theta$ is very close to an integer.

\subsection{A key technical result}

The following lemma is at the heart of our lower bound.  It shows that if $e^\theta$ is close to an integer, then any partition $\mu$ of $n \geq m$ which extends a particular partition $\lambda_{\mathsf{sq}}$ of $m$ must be such that the Fourier coefficient $\widehat{{\cal M}_\theta}(\rho_\mu)$ of Cayley-Mallows noise has small singular values.\ignore{ \rnote{If the main take-away of this section is~\Cref{thm:lower-bound} then we don't really need the business about $n \geq m$, it would be enough to just have the $n=m$ case.  The more general statement is scarecely more complex but strictly speaking we don't need the added complexity; should we keep it? }
}

 \begin{lemma}~\label{lem:upper-bound-singular-1}
 Let $\lambda_{\mathsf{sq}}$ denote the partition $(t,\dots,t)$ of $m=t(t+j)$ whose Young tableau is a rectangle with $t+j$ rows and $t$ columns. Let $\theta>0$ be such that $\big| e^\theta-j \big| \le \eta$ where $\eta \leq 1/2.$
 Let $n \ge m$, $\mu \vdash n$ and $\lambda_{\mathsf{sq}} \Uparrow \mu$ (recall~\Cref{def:Young-lattice}).
  Then \[
\widehat{\mathcal{M}_{\theta}}(\rho_{\mu})= c_{\mu, \theta} \cdot \mathsf{Id},
\quad \text{where~}c_{\mu,\theta} \le \eta^t.
\]
Here $\rho_\mu$ denotes the irreducible representation of $\mathbb{S}_n$ corresponding to the partition $\mu$.
 \end{lemma}
 \begin{proof}
Let $\mu=(\mu_1,\dots,\mu_r)$.  By Lemma~\ref{lem:eigen-mallows}, we have that 
\[
\widehat{\mathcal{M}_{\theta}}(\rho_{\mu})= c_{\mu, \theta} \cdot \mathsf{Id}, 
\]
where~\Cref{eq:simplified} gives the precise value of $c_{\mu, \theta}$ as 
\begin{equation}~\label{eq:upper-bound-2}
c_{\mu, \theta} = \frac{\prod_{u \in \mathcal{A}} (q+c(u))}{\prod_{u \in {\cal B}} (q+c(u))},
\quad \text{where} \quad q=e^\theta.
\end{equation}
Here $\mathcal{A}$ denotes the set of cells of the Young diagram of $\mu$ which are not in the first row and ${\cal B}$ denotes the rightmost $n-\mu_1$ many cells in the Young diagram of the trivial partition $\mathsf{Triv}_n=(n)$.  Note that in this lemma, we are  trying to upper bound~\Cref{eq:upper-bound-2} whereas Lemma~\ref{lem:eigen-mallows} was about lower bounding this quantity. 

To upper bound~\Cref{eq:upper-bound-2}, we first observe that there is an obvious bijection $\Phi: \mathcal{A} \rightarrow \mathcal{B}$ such that if $\Phi(u) =v$, then $c(v)>|c(u)|>0$.

Next, let $\mathcal{A}_{-j} \subset {\cal A}$ be ${\cal A} := \{(r,s): s-r=j 
\textrm{ and } (r,s) \in \mathcal{A}\}$. Since $\lambda_{\mathsf{sq}} \Uparrow \mu$, 
it follows that $|\mathcal{A}_{-j}| \ge t$. 
As a result, we can upper bound $c_{\mu, \theta}$ as follows:
\begin{align*}
c_{\mu,\theta} = \frac{\prod_{u \in \mathcal{A}} (q+c(u))}{\prod_{u \in \mathcal{B}} (q+c(u))} &= \prod_{u \in \mathcal{A}} \frac{q+c(u)}{q +c(\Phi(u))} = \left(\prod_{u \in \mathcal{A}_{-j}} \frac{q+c(u)}{q +c(\Phi(u))}\right)
\left(\prod_{u \in \mathcal{A} \setminus\mathcal{A}_{-j}} \frac{q+c(u)}{q +c(\Phi(u))}\right)\\
&\le \prod_{u \in \mathcal{A}_{-j}}q+c(u) \tag{using $c(\Phi(u))>|c(u)|>0$ and $q>0$} \\
&\le \eta^{t}. 
\end{align*} 
 \end{proof}

\subsection{Proof of~\Cref{thm:lower-bound-intro}}

\Cref{thm:lower-bound-intro} is an immediate consequence of the following result.  It shows that if $e^\theta$ is close to an integer $j$, then it may be statistically impossible to learn a distribution $f$ supported on $k$ rankings without using many samples from ${\cal M}_\theta \ast f$:

\begin{theorem}~\label{thm:lower-bound}
Given $j \in \N$, there are infinitely many values of $k$ and $m =m(k) \approx {\frac {\log k}{\log \log k}}$ such that the following holds:    there are two distributions $f_1,f_2$ over $\mathbb{S}_m$ with the following properties:
\begin{enumerate}
\item $\dtv(f_1,f_2)=1$ (i.e. the distributions $f_1$ and $f_2$ have disjoint support);
\item $|\supp(f_1)|,|\supp(f_2)| \leq k$;
\item For any $\theta>0$ such that $|e^\theta-j|\leq \eta \leq 1/2$, we have that
$\dtv(
\mathcal{M}_{\theta} \ast f_1,
\mathcal{M}_{\theta} \ast f_2) \le 2 \cdot \eta^{\Theta\left(\sqrt{\frac {\log k}{\log \log k}}\right)}.$
\end{enumerate}
\end{theorem}

\ignore{
}

\begin{proof}
Let $t \geq j$ be any integer, let $m=t(t+j)$, and let $k=m!$.  We first construct the two distributions $f_1,f_2$ over $\mathbb{S}_m$ and argue that properties (1) and (2) hold.

Let $\lambda_{\mathsf{sq}} \vdash m$ be the partition whose Young tableau is a rectangle with $t+j$ rows and $t$ columns. 
Let us  consider the character $\chi_{\mathsf{sq}}: \mathbb{S}_m \rightarrow \mathbb{Q}$  corresponding to the partition $\lambda_{\mathsf{sq}}$.  By~\Cref{fact:rational} we have that $\chi_{\mathsf{sq}}$ is rational valued, and by~\Cref{thm:characters} we have that $\sum_{\sigma \in \mathbb{S}_n} \chi_{\mathsf{sq}}(\sigma) =0$. Thus, we have that 
\begin{equation}~\label{eq:normalization}
\sum_{\sigma  \in \mathbb{S}_n} |\chi_{\mathsf{sq}}(\sigma)| \cdot \mathbf{1}_{\chi_{\mathsf{sq}}(\sigma)>0} =\sum_{\sigma  \in \mathbb{S}_n} |\chi_{\mathsf{sq}}(\sigma)| \cdot \mathbf{1}_{\chi_{\mathsf{sq}}(\sigma)<0} =: C_{\mathsf{sq}}
\end{equation}
for some $C_{\mathsf{sq}}$ (which is nonzero again by~\Cref{thm:characters}).
We now define distributions $f_1$ and $f_2$ over $\mathbb{S}_m$ as 
\[
f_1(\sigma) = \begin{cases} \frac{1}{C_{\mathsf{sq}}} \cdot \chi_{ \mathsf{sq}}(\sigma) 
&\textrm{ if }\chi_{ \mathsf{sq}}(\sigma)>0 \\
0 &\textrm{ otherwise},
\end{cases} \quad \quad  f_2(\sigma) = \begin{cases} \frac{-1}{C_{\mathsf{sq}}} \cdot \chi_{ \mathsf{sq}}(\sigma) 
&\textrm{ if }\chi_{ \mathsf{sq}}(\sigma)<0 \\
0 &\textrm{ otherwise}.
\end{cases}
\]

From their definitions and~\Cref{eq:normalization} it is immediate that $f_1$ and $f_2$ are distributions over $\mathbb{S}_m$ which have disjoint support.  Since $|\mathbb{S}_m|=k$, this gives items $1$ and $2$ of the theorem.

To prove the third item, observe (recalling the comment immediately after~\Cref{def:characters}) that the function $g: \mathbb{S}_m \rightarrow \mathbb{C}$, defined as $g(\sigma) := f_1(\sigma) - f_2(\sigma)= \frac{1}{C_{\mathsf{sq}}} \cdot \chi_{ \mathsf{sq}}(\sigma)$, is a class function. Choose any partition $\lambda \vdash m$ and the corresponding irreducible representation $\rho_{\lambda}$ of $\mathbb{S}_m$. By applying Lemma~\ref{lem:class-func}, we have that
\begin{equation}
\widehat{g}(\rho_{\lambda}) = c_{\lambda} \cdot \mathsf{Id} \ \ \textrm{ where }  \ \ c_{\lambda} = \frac{\sum_{\sigma \in \mathbb{S}_m} g(\sigma) \cdot \chi_{\lambda}(\sigma)}{\dim(\rho_{\lambda})} .\label{eq:g}
\end{equation}

We analyze the multiplier $c_\lambda$ by noting that
\begin{eqnarray}
 c_{\lambda} = \frac{\sum_{\sigma \in \mathbb{S}_m} g(\sigma) \cdot \chi_{\lambda}(\sigma)}{\dim(\rho_{\lambda})} &=&  \frac{\sum_{\sigma \in \mathbb{S}_m} \chi_{\mathsf{sq}}(\sigma) \cdot \chi_{\lambda}(\sigma)}{\dim(\rho_{\lambda}) \cdot C_{\mathsf{sq}}} \nonumber  \\
 &=& \frac{m! \cdot \Ind[\lambda = \lambda_{\mathsf{sq}}]}{\dim(\rho_{\lambda}) \cdot C_{\mathsf{sq}}} \ \ \textrm{using~\Cref{thm:characters}}. \label{eq:clambda}
\end{eqnarray} 
Thus, we have
\begin{align}
\Vert \mathcal{M}_{\theta} \ast f_1 - \mathcal{M}_{\theta} \ast f_2 \Vert_1 
&=
\sum_{\sigma \in \mathbb{S}_m} |{\cal M}_\theta \ast f_1(\sigma) - {\cal M}_\theta \ast f_2(\sigma)| \nonumber \\
&=\sum_{\sigma \in \mathbb{S}_m} |{\cal M}_\theta \ast g(\sigma)| \nonumber \tag{linearity and $g=f_1 - f_2$}\\
&=
{\frac 1 {m!}}
\sum_{\sigma \in \mathbb{S}_m} \left|
 \sum_{\mu \vdash m}  \dim(\rho_\mu) \mathsf{Tr}[\widehat{{{\cal M}_\theta} \ast g}(\rho_\mu) \rho_\mu(\sigma^{-1})]
\right| \tag{\Cref{def:Fourier}, inverse Fourier transform of ${\cal M}_\theta \ast g$}\\
&= 
{\frac 1 {m!}}
\sum_{\sigma \in \mathbb{S}_m} \left|
 \sum_{\mu \vdash m}  \dim(\rho_\mu) \mathsf{Tr}[\widehat{{\cal M}_\theta}(\rho_\mu)  \widehat{g}(\rho_\mu) \rho_\mu(\sigma^{-1})]
\right| \tag{convolution identity}\\
&= 
{\frac 1 {\dim(\rho_{\lambda_{\mathsf{sq}}}) \cdot C_{\mathsf{sq}}}}
\sum_{\sigma \in \mathbb{S}_m} \left|
   \dim(\rho_{\lambda_{\mathsf{sq}}}) \mathsf{Tr}[\widehat{\cal M_\theta}(\rho_{\lambda_{\mathsf{sq}}})  \rho_{\lambda_{\mathsf{sq}}}(\sigma^{-1})]
\right| \tag{Equations~\ref{eq:g} and~\ref{eq:clambda}}\\
&={\frac 1 { C_{\mathsf{sq}}}}
\sum_{\sigma \in \mathbb{S}_m} \left|
    \mathsf{Tr}[\widehat{\cal M_\theta}(\rho_{\lambda_{\mathsf{sq}}})  \rho_{\lambda_{\mathsf{sq}}}(\sigma^{-1})]
\right|
\end{align}

To deal with $\widehat{\cal M_\theta}(\rho_{\lambda_{\mathsf{sq}}})$, we apply~\Cref{lem:upper-bound-singular-1}. In particular, by setting $n=m$ and $\mu = \lambda_{\mathsf{sq}}$ in~\Cref{lem:upper-bound-singular-1}, we get that 
\[
\widehat{\mathcal{M}_{\theta}}(\rho_{\lambda_{\mathsf{sq}}})= c_{\lambda_{\mathsf{sq}}, \theta} \cdot \mathsf{Id} ,
\]
where $|c_{\lambda_{\mathsf{sq}}, \theta}| \le \eta^t$, and we thus get that

\begin{align}
\Vert \mathcal{M}_{\theta} \ast f_1 - \mathcal{M}_{\theta} \ast f_2 \Vert_1 \le \frac{\eta^t}{C_{\mathsf{sq}}} \cdot \sum_{\sigma \in \mathbb{S}_m} 
\left|
 \mathsf{Tr}[ \rho_{\lambda_{\mathsf{sq}}}(\sigma^{-1})
\right|
&=
 \frac{\eta^t}{C_{\mathsf{sq}}} \cdot \sum_{\sigma \in \mathbb{S}_m} 
\left|
\chi_{{\mathsf{sq}}}(\sigma^{-1})
\right|. \label{eq:almost}
\end{align}
Finally, recalling that 
\[
C_{\mathsf{sq}} = \frac{\sum_{\sigma \in \mathbb{S}_n} |\chi_{\mathsf{sq}}(\sigma)|}{2},
\]
we get that the RHS of~\Cref{eq:almost} is $2 \eta^t$. Recalling that $t\geq\sqrt{m/2}$, the theorem is proved. 
\end{proof}

\ignore{

\subsection{Proof of~\Cref{thm:lb-Mallows}} \label{sec:eightpointthree}

\ignore{
}

THIS SUBSECTION NEEDS TO BE FINISHED IF WE WANT TO HAVE A THEOREM ABOUT GENERAL $n$

\Cref{thm:lower-bound} only deals with distributions over $\mathbb{S}_m$ where $m \approx {\frac {\log k}{\log \log k}}$.  We now generalize that result to hold for distributions over $\mathbb{S}_n$ for arbitrary $n$:

\begin{theorem}~\label{thm:lower-bound2}
Let $j \in \N$ and let $\theta>0$ be such that $|e^\theta-j|\leq \eta < 1/2.$  Then there are infinitely many values of $k$ and $m =m(k) \approx {\frac {\log k}{\log \log k}}$ such that the following holds: \red{for all $n \geq m$},  there are two distributions $f'_1,f'_2$ over \red{$\mathbb{S}_n$} with the following properties:
\begin{enumerate}
\item $\dtv(f'_1,f'_2)=1$ (i.e. the distributions $f'_1$ and $f'_2$ have disjoint support);
\item $|\supp(f_1)|,|\supp(f_2)| \leq k$;
\item $\dtv(
\mathcal{M}_{\theta} \ast f_1,
\mathcal{M}_{\theta} \ast f_2) \le 2 \cdot \eta^{\Theta\left(\sqrt{\frac {\log k}{\log \log k}}\right)}.$
\end{enumerate}
\end{theorem}

\begin{proof}
Choose $k$ and the two distributions $f_1$, $f_2$ supported on $\mathbb{S}_m$ (for $m \approx \frac{\log k}{\log \log k}$) from Theorem~\ref{thm:lower-bound}, where $m = t(t+j)$ for $t \ge j$. We recall that from Equations~\ref{eq:g} and \ref{eq:clambda}, writing $g=f_1-f_2$, we have 
\begin{equation}~\label{eq:construct}
\widehat{g}(\rho_{\lambda_{}}) = c_{\lambda} \cdot \mathsf{Id} \textrm{ where } \   c_{\lambda}= \frac{m! \cdot \mathbf{1}(\lambda = \lambda_{\mathsf{sq}})}{\dim(\rho_{\lambda}) \cdot C_{\mathsf{sq}}}.
\end{equation}
We now extend the distributions $f_1,f_2$ over $\mathbb{S}_m$ to distributions $f'_1,f'_2$ over $\mathbb{S}_n$ in the obvious way, by identifying $\mathbb{S}_m$ with the subgroup of $\mathbb{S}_n$ which keeps the elements $\{m+1, \ldots, n\}$ fixed, and we define $g': \mathbb{S}_n \to \R$ as $g'=f'_1-f'_2.$  To analyze these distributions, it is helpful for us to classify the Young diagrams of $[n]$ according to their locations in Young's lattice with respect to $\lambda_{\mathsf{sq}}$. In particular, we define
$\mathcal{L}_{u}$ and 
$\mathcal{L}_{n}$ as
\[
\mathcal{L}_{u} = \{\mu \vdash n : \lambda_{\mathsf{sq}} \Uparrow \mu\} \quad \quad \mathcal{L}_{n} = \{\mu \vdash n : \lambda_{\mathsf{sq}} \not \Uparrow \mu\}
\]
 \Cref{eq:construct} and \Cref{thm:branch} together implies that 
 \begin{equation}~\label{eq:zero-Fourier-1}
 \widehat{g'}(\rho_{\mu}) =0 \ \  \ \  \ \ \textrm{for }\mu \in \mathcal{L}_{n}. 
 \end{equation}
 To bound $\widehat{\mathcal{M}_{\theta}}(\rho_{\mu})$ for $\mu \in \mathcal{L}_u$, we apply Lemma~\ref{lem:upper-bound-singular-1} to get 
 \begin{equation}~\label{eq:upper-bound-3}
 \widehat{\mathcal{M}_{\theta}}(\rho_{\mu}) =c_{\mu,\theta} \cdot \mathsf{Id} \ \textrm{where} \ \ |c_{\mu,\theta}| \le \delta^t. 
 \end{equation}
 Consequently, we have

 \begin{align}
\Vert \mathcal{M}_{\theta} \ast f'_1 - \mathcal{M}_{\theta} \ast f'_2 \Vert_1  &=
 \sum_{\sigma \in \mathbb{S}_n} |{\cal M}_\theta \ast f'_1(\sigma) - {\cal M}_\theta \ast f'_2(\sigma)| \nonumber \\
 &=\sum_{\sigma \in \mathbb{S}_n} |{\cal M}_\theta \ast g'(\sigma)| \nonumber \tag{linearity and $g'=f'_1 - f'_2$}\\
&=
{\frac 1 {n!}}
\sum_{\sigma \in \mathbb{S}_n} \left|
 \sum_{\mu \vdash n}  \dim(\rho_\mu) \mathsf{Tr}[\widehat{{{\cal M}_\theta} \ast g'}(\rho_\mu) \rho_\mu(\sigma^{-1})]
\right| \tag{\Cref{def:Fourier}, inverse Fourier transform of ${\cal M}_\theta \ast g'$}\\
&=
{\frac 1 {n!}}
\sum_{\sigma \in \mathbb{S}_n} \left|
 \sum_{\mu \vdash n}  \dim(\rho_\mu) \mathsf{Tr}[\widehat{{\cal M}_\theta}(\rho_\mu)  \widehat{g'}(\rho_\mu) \rho_\mu(\sigma^{-1})]
\right| \tag{convolution identity}\\
 \end{align}

 \gray{
 OLD STUFF
 \begin{eqnarray}
\Vert \mathcal{M}_{\theta} \ast f_1 - \mathcal{M}_{\theta} \ast f_2 \Vert_1  &\le& \sqrt{n!} \cdot \Vert \mathcal{M}_{\theta} \ast f_1 - \mathcal{M}_{\theta} \ast f_2 \Vert_2 \nonumber \\
&\leq& \sqrt{
\sum_{\mu \vdash n} \Vert 
\widehat{\mathcal{M}_{\theta} \ast f_1}(\rho_{\mu}) - \widehat{\mathcal{M}_{\theta} \ast f_2}(\rho_{\mu}) 
 \Vert_F^2} \quad \textrm{using Definition~\ref{def:Fourier}} \nonumber \\
 &=& \sqrt{
\sum_{\mu \vdash n} \Vert 
\widehat{\mathcal{M}_{\theta}}(\rho_{\mu}) \cdot  \widehat{f_1- f_2} (\rho_{\mu}) 
 \Vert_F^2}  \nonumber \\
 &=& \sqrt{
\sum_{\mu  \in \mathcal{L}_{u}} \Vert 
\widehat{\mathcal{M}_{\theta}}(\rho_{\mu}) \cdot  \widehat{f_1- f_2} (\rho_{\mu}) 
 \Vert_F^2} \nonumber \quad \textrm{using (\ref{eq:zero-Fourier-1})} \\
 &\le & \sqrt{
\sum_{\mu  \in \mathcal{L}_{u}} \Vert 
\widehat{\mathcal{M}_{\theta}}(\rho_{\mu})
\Vert_{2}^2 
 \cdot \Vert \widehat{f_1- f_2} (\rho_{\mu}) \Vert_F^2}\nonumber \\
 &\le&
 \delta^t  \sqrt{
\sum_{\mu  \in \mathcal{L}_{u}} \Vert \widehat{f_1- f_2} (\rho_{\mu}) 
 \Vert_F^2} \ \ \textrm{using (\ref{eq:upper-bound-3})} \nonumber \\
 &\le& \sqrt{n!} \cdot \delta^t \cdot 
 \end{eqnarray}
 }
\end{proof}

}

\ignore{

\section{Algorithms and lower bounds for learning from partial rankings}~\label{sec:partial-ranking}

\red{rewrite this section so it doesn't talk about the $\eta$-dominant condition}

\gray{

In this section, we extend our algorithms to the setting where we only get partial rankings from a mixture model. While incomparable to the symmetric noise model and the Generalized Mallows model considered earlier in the paper, the analysis here turns out to be significantly simpler. We start by formally defining the partial ranking noise model. 
To formally define this model, for any $1 \le i \le n$, let $i^1$ be the string ``$i$" and $i^0$ be the empty string. 
\begin{definition}~\label{def:partial-noise}
A partial noise on $\mathbb{S}_n$ is specified by a non-negative vector $\mathbf{p} = (\mathbf{p}_0, \ldots, \mathbf{p}_n)$ such that $\sum_{j \ge 0} \mathbf{p}_j =1$. Given a probability distribution $\mathcal{D}$ over $\mathbb{S}_n$, we define $\partn{\mathbf{p}}{\mathcal{D}}$ as a distribution over ordered tuples of $[n]$ defined as follows: 
\begin{enumerate}
\item Choose $0 \leq j \leq n$ with probability $\mathbf{p}_j$ and choose a uniformly random subset $S$ of $ [n]$ of size $j$. 
\item Sample $\sigma \sim \mathcal{D}$ and output the string sequence $\sigma(1)^{1 \in S} \cdot \ldots \cdot
\sigma(n)^{n \in S}$. 
\end{enumerate}
In other words, we choose a permutation (or ranking) $\sigma \sim \mathcal{D}$ and output the ranking $\sigma(1) \cdot \ldots \cdot \sigma(n)$ dropping all the positions which are not in $S$. 
\end{definition}
Before we state the main theorem of this section, we need more definition. 
\begin{definition}~\label{def:upward}
Let $\mathbf{p} = (\mathbf{p}_0, \ldots, \mathbf{p}_n)$ (such that $\sum_{j \ge 0} \mathbf{p}_j =1$) be a non-negative vector and define the distribution $\mathcal{D}_{\mathbf{p}}$ on $\{0,1\}^n$ as follows: Choose $0 \le j \le n$ with probability $\mathbf{p}_j$; Now sample a uniformly random subset $S \subseteq [n]$ of size $j$. We say that $\mathbf{p}$ is $\eta$-dominant for sets of size $t$ if for any subset $S'$ of size $t$, 
\[
\Pr_{x \sim \mathcal{D}_{\mathbf{p}}} [\forall j \in S': x_j=1] \ge \eta. 
\]
\end{definition}

We now state the main theorem of this section. 
\begin{theorem}~\label{thm:main-partial-rank}
There is an algorithm with the following guarantee: Let $f$ be an unknown distribution over $\mathbb{S}_n$ such that $\mathsf{supp}(f) \le k$ and for every $x$ such
that $f(x)>0$, $f(x) \ge \epsilon$. Suppose $\mathbf{p}$ is a vector such that we get access to random samples
from $\partn{\mathbf{p}}{f}$. Further, $\mathbf{p}$ is $\eta$-dominant for sets of size $2\log k$. Given any error parameter $\delta>0$, the algorithm runs in time $\mathsf{poly}(n,k^{\log k}, 1/\eta, 1/\epsilon, 1/\delta)$ and outputs a distribution $g$ over $\mathbb{S}_n$ such that $\Vert f - g\Vert_1 \le \delta$. 
\end{theorem}
\begin{proof}
Consider the $m=\binom{n}{2}$ unordered pairs of distinct elements from $[n]$ and choose any total order for these pairs. Once this order is fixed, every permutation $\sigma \in \mathbb{S}_n$ can be uniquely mapped to an element of $\{0,1\}^m$ as follows: $\mathsf{Bin}: \sigma \mapsto x$ such that if $(i,j)$ is the pair corresponding to coordinate $\ell$ (where $i<j$), then $x_{\ell}=0$ if $\sigma(i) < \sigma(j)$ and $1$ otherwise. Applying this mapping element by element, every distribution over $\mathbb{S}_n$ also maps uniquely to a distribution over $\{0,1\}^m$. 

Let us now define $f^{\mathsf{bin}}$ denote the image of $f$ under the map. Note that $|\mathsf{supp}(f^{\mathsf{bin}})| \le k$ and that for every $x$ such that $f^{\mathsf{bin}}(x)>0$, $f^{\mathsf{bin}}(x) >\epsilon$. Next we make the following claim: 
\begin{claim}~\label{clm:partial-compute}
Let $J$ be any set of size at most $\log k$. Then, for any error $\kappa>0$, in time $\mathsf{poly}(n,k, 1/\kappa, 1/\eta)$, we can compute (up to error $\pm \kappa$)
$
\Pr_{x \sim f^{\mathsf{bin}}}[\forall j \in J \ x_j=1]
$. 
\end{claim}
\begin{proof}
For any $j \in J$, there is a pair of distinct elements in $(r,s)$ and $\mathsf{Bin}(r,s) =j$. Let us define $\mathsf{Bin}^{-1}(J) = \cup_{j \in  J: \mathsf{Bin}((r,s))=j} (r,s)$. Clearly, $|\mathsf{Bin}^{-1}(J) | \le 2|J|$. We now consider the following estimator for
$
\Pr_{x \sim f^{\mathsf{bin}}}[\forall j \in J \ x_j=1]
$.  
\begin{enumerate}
\item Set $\mathsf{Count-Total}=0$ and $\mathsf{Count-pos}=0$. 
\item Repeat $T = \frac{2}{\kappa^2 \cdot \eta}$ times. 
\item \hspace*{7pt} Sample $\gamma \sim \partn{\mathbf{p}}{f}$. 
If $\mathsf{Bin}^{-1}(J)$ is not contained in $\gamma$, discard the sample. 
\item \hspace*{7pt} Else,  $\mathsf{Count-Total} +=1$; If $\forall j \in J$, $(r,s) = \mathsf{Bin}^{-1}(j)$ and $r$ appears before $s$ in $\gamma$, then $\mathsf{Count-pos} +=1$. 
\item Output $\frac{\mathsf{Count-pos}}{\mathsf{Count-Total}}$. 
\end{enumerate}
To analyze this estimator, let us define $\mathcal{E}$ to be the event that the sample $\gamma$ is not discarded, i.e., $\mathsf{Bin}^{-1}(J)$ is  contained in $\gamma$. We now make the following claims. 
\begin{claim}~\label{clm:event-bound-1}
For $\gamma \sim \partn{\mathbf{p}}{f}$, $\Pr[\mathcal{E}] \ge \eta$. 
\end{claim}
\begin{proof}
Note that $|\mathsf{Bin}^{-1}(J)| \le 2 \log k$. Since $\mathbf{p}$ is $\eta$-dominating for sets of $2 \log k$, with probability at least $\eta$, the set $\mathsf{Bin}^{-1}(J)$ appears in $\gamma$ finishing the proof. 
\end{proof}
\begin{claim}~\label{clm:empirical}
Conditioned on $\mathcal{E}$, 
\[
\Pr[\forall j \in J, (r,s) = \mathsf{Bin}^{-1}(j): \textrm{r appears before s}] = \Pr_{x \sim f^{\mathsf{bin}}}[\forall j \in J \ x_j=1]. 
\]
\end{claim}
\begin{proof}
This proof is obvious but the notation needs to be improved substantially -- otherwise, the proof reads like crap. So, I will do this later. 
\end{proof}
Combining Claims~\ref{clm:event-bound-1} and \ref{clm:empirical}, we immediately get Claim~\ref{clm:partial-compute}. 
\end{proof}
Once we have Claim~\ref{clm:partial-compute}, we can as a black-box, apply the Wigderson-Yehudayoff algorithm to recover  $f^{\mathsf{bin}}$ (and consequently $f$). 
\end{proof}

\red{State the corresponding lower bound. Prove it.}

}
}

\newpage
\appendix

\section{Basics of representation theory over the symmetric group}  \label{sec:rep}
Representation theory of the symmetric group $\mathbb{S}_n$ is at the technical core of this paper. In this appendix we briefly review the definitions and results that we require, starting first with general groups and then specializing to $\mathbb{S}_n$ as necessary.  See Curtis and Reiner~\cite{curtis1966representation} (or many other sources) for an extensive reference on representation theory of finite groups
and James~\cite{james2006representation} or M{\'e}liot~\cite{meliot2017representation} for an extensive reference on representation theory of $S_n$.

\subsection{General groups} \label{sec:repgeneral}

We start by recalling the definition of a representation:

\begin{definition}~\label{def:representation}
	For any group $G$, a \emph{representation} $\rho: G \rightarrow \mathbb{C}^{m \times m}$ is a group homomorphism,  i.e.~a function from $G$ to $\mathbb{C}^{m \times m}$ that satisfies $\rho(g) \cdot \rho(h) = \rho(g \cdot h)$ for all $g,h \in G.$ The \emph{dimension} of such a representation $\rho$ is $m$.
\end{definition}
In this paper, unless otherwise mentioned, all  representations $\rho$  are \emph{unitary} -- in other words, for every $g \in G$, $\rho(g)$ is a unitary matrix. Over finite groups, any representation can be made unitary by applying a similarity transformation; by this we mean that if $\rho$ is a representation, then there is an invertible matrix $Z$ such that the new map $\tilde{\rho}$ defined as $\tilde{\rho} (g) = Z^{-1} \cdot \rho(g) \cdot Z$ is a unitary representation. (The reader should verify that as long as $Z$ is invertible, the map $\tilde{\rho}$ is always a representation if $\rho$ is a representation.) 
Two such representations $\rho$ and $\tilde{\rho}$ are said to be \emph{equivalent}.

Next we recall the notion of an irreducible representation:

\begin{definition}~\label{def:irred}
A representation $\rho: G \rightarrow \mathbb{C}^{m \times m}$ is said to be \emph{reducible} if there exists a proper subspace $V$ of $\mathbb{C}^{m}$ such that $\rho (g) \cdot V \subseteq V$ for all $g \in G$. If there is no such proper subspace $V$, then $\rho$ is said to be \emph{irreducible}. 
\end{definition}

It is well known that any finite group has only finitely many irreducible representations, up to the above notion of equivalence, and that every representation of a finite group $G$ can be written as a direct sum of irreducible representations:

\begin{theorem} [Maschke's theorem, see e.g.~Theorem 1.3 of \cite{meliot2017representation}] \label{thm:Maschke}
For $G$ a finite group, there is a finite set of distinct irreducible representations $\{\rho_1, \ldots, \rho_r\}$ such that for any representation $\rho: G \rightarrow \mathbb{C}^{m \times m}$, there is a invertible transformation $Z \in \mathbb{C}^{m \times m}$ such that  $Z^{-1}  \rho Z$ is block diagonal where each block is one of $\{\rho_1, \ldots, \rho_r\}$. In other words, $Z^{-1} \rho Z$ is equal to the direct sum $\oplus_{\ell=1}^M \mu_\ell$ where each $\mu_\ell$ is an element of $\{\rho_1, \ldots, \rho_r\}$.
\end{theorem}

We remind the reader that elements $g,h$ in a group $G$ are said to be \emph{conjugates} if there is an element $t \in G$ such that $tgt^{-1}=h.$  Define $\mathsf{Cl}(g)$, the \emph{conjugacy class} of $g$, to be $\{h: h \textrm{ is conjugate to } g\}$; it is easy to see that the different conjugacy classes form a partition of $G$.  

We recall some very standard facts about irreducible representations: 

\begin{theorem} [see e.g. Theorem~2.3.1 of \cite{Wigderson:representationtheory}] \label{thm:irrep-facts} 
Let $G$ be a finite group and let $\{\rho_1, \ldots, \rho_r\}$ be the set of its irreducible representations, where $\rho_i: G \rightarrow \mathbb{C}^{d_i \times d_i}$. Then 
\begin{enumerate}
\item $\sum_{i=1}^r d_i^2 = |G|$. 
\item The number of conjugacy classes is equal to $r$, the number of distinct irreducible representations.
\item For $1 \le s,t \le d_i$, let $\rho_{i,s,t}: G \rightarrow \mathbb{C}$  be the $(s,t)$ entry of $\rho_i(g)$. Then,  for $1 \le i_1, i_2 \le r$, $1 \le s_1,t_1 \le d_{i_1}$ and $1 \le s_2,t_2 \le d_{i_2}$
\[
\Ex_{g \in G}[\rho_{i_1, s_1, t_1}(g) \cdot \overline{\rho_{i_2, s_2, t_2}(g)}] = \begin{cases} \frac{1}{d_{i_1}} &\textrm{if }i_1=i_2, \ s_1=s_2 \textrm{ and } t_1=t_2 \\
0&\textrm{ otherwise} \end{cases}
\]
\item The representations $\rho_1, \ldots, \rho_r$ are unitary. 
\end{enumerate}
\end{theorem}
A restatement of (3) above is that the functions $\{\rho_{i,s,t}(\cdot)\}$ are orthogonal. Combining this with $\sum_{i=1}^r d_i^2 = |G|$ (given by (1)), we get that the functions $\{\rho_{i,s,t}\}_{1\le i\le r, 1 \le s,t \le d_i}$ form an orthogonal basis for $\mathbb{C}^{G}$. 

With an orthonormal basis for the set of complex-valued functions on $G$ in hand (in other words, a basis for the group algebra $\mathbb{C}[G]$), we are ready to define the \emph{Fourier transform} of a function $f: G \rightarrow \mathbb{C}$:

\begin{definition}~\label{def:Fourier}
Let $G$ be a finite group with irreducible representations given by $\{\rho_1, \ldots, \rho_r\}$ and let $f: G \to \mathbb{C}.$  The \emph{Fourier transform} of $f$ is given by matrices $\widehat{f}(\rho_1), \ldots, \widehat{f}(\rho_r)$, where 
$$
\widehat{f}(\rho_i) = \sum_{g \in G} f(g) \cdot \rho_i(g). 
$$
The inverse transform is given by
\[
f(g) = {\frac 1 {|G|}} \sum_{i=1}^r \dim(\rho_i) \mathsf{Tr}[\widehat{f}(\rho_i) \rho_i(g^{-1})].
\]
\end{definition}

\emph{Parseval's identity} states that for any $f$ as above, we have 
\begin{equation}~\label{eq:Parseval}
\sum_{i=1}^r \Vert \widehat{f}(\rho_i) \Vert_F^2 = |G| \cdot  \sum_{g \in G} |f(g)|^2.
\end{equation}

We next recall the definition of characters and class functions for a group $G$. 
\begin{definition}~\label{def:class-function}
Given a finite group $G$, a function $f: G \rightarrow \mathbb{C}$ is said to be a \emph{class function of $G$} if $f(g)$ only depends on the conjugacy class of $g$, i.e.~$f(g) = f(hgh^{-1})$ for every $h \in G$.
\end{definition}

\begin{definition}~\label{def:characters}
The \emph{character} $\chi_\rho: G \rightarrow \mathbb{C}$ corresponding to a representation $\rho: G \rightarrow \mathbb{C}^{m \times m}$ is given by $\chi_\rho(g) := \mathsf{Tr}(\rho(g))$. 
\end{definition}
We observe that $\chi_\rho(\cdot)$ is a class function of $G$, and that if $\rho$ and $\tilde{\rho}$ are unitarily equivalent, then $\chi_\rho(\cdot) = \chi_{\tilde{\rho}}(\cdot)$. 
We recall some standard facts about characters and class functions:

\begin{theorem}~\label{thm:characters}
Let $G$ be a finite group and let $\{\rho_1, \ldots, \rho_r\}$ be its set of irreducible representations. Let $\chi_{\rho_1}, \ldots, \chi_{\rho_r}$ be the corresponding characters. Then we have: 
\begin{enumerate}
\item ~ [Schur's lemma] $\mathbf{E}_{g \in G}[\chi_{\rho_i}(g) \cdot \overline{\chi_{\rho_j}(g)}] = \delta_{i,j}$. 
\item The functions $\{\chi_{\rho_i}(\cdot)\}_{1 \le i \le r}$ forms an orthonormal basis for all class functions of $G$.  
\end{enumerate}
\end{theorem}

The following (standard) claim shows that the Fourier transform of any class function $f$ is a diagonal matrix (in fact, a scalar multiple of the identity matrix):
\begin{lemma}~\label{lem:class-func}
Let $f: G \rightarrow \mathbb{C}$ be a class function and let $\rho : G \rightarrow \mathbb{C}^{m \times m}$ be an irreducible representation of $G$. Then $\widehat{f}(\rho) = c \cdot \mathsf{Id}$ where $c = \frac{\sum_{g \in G} f(g) \chi_{\rho}(g)}{m}$ and $\mathsf{Id}$ is the identity matrix.
\end{lemma}
\begin{proof}
Choose any $h \in G,$ and observe that 
\begin{eqnarray*}
\rho(h) \cdot \widehat{f}(\rho)  &=&  \rho (h) \cdot \big(
\sum_{g \in G} f(g) \rho(g) \big) \\
&=&  \rho (h) \cdot \big(
\sum_{g \in G} f(h^{-1}  g  h) \rho(h^{-1}  g  h) \big)= \rho (h) \cdot \big(
\sum_{g \in G} f(  g  ) \rho(h^{-1}  g  h) \big) \\
&=& \rho(h) \cdot \rho(h^{-1}) \cdot \big(
\sum_{g \in G} f(  g  ) \rho(g) \big) \cdot \rho(h) = \widehat{f}(\rho) \cdot \rho(h). 
\end{eqnarray*}
As a consequence of Schur's lemma, we have that if a matrix $A$ is such that $A \cdot \rho(h) = \rho(h) \cdot A$ for all $h \in G$, then $A = c \cdot \mathsf{Id}$. Thus, we get that $\widehat{f}(\rho) =c \cdot \mathsf{Id}$. The lemma follows by taking trace on both sides. 
\end{proof}

\subsection{Representation theory of the symmetric group} \label{sec:repsym}

Representation theory of the symmetric group has many applications to algebra, combinatorics and statistical physics and has been intensively studied (as mentioned earlier, see e.g.~\cite{james2006representation,meliot2017representation} for detailed treatments).  Below we only recall a few basics which we will need.

The first notion we require is that of a \emph{Young diagram.}
Consider a partition $\lambda  = (\lambda_1, \ldots, \lambda_k)$ of $n$ where $\lambda_1 \ge \lambda_2 \ge \ldots \ge  \lambda_k>0$ and $\lambda_1 + \ldots + \lambda_k=n$. We indicate that $\lambda$ is such a partition by writing ``$\lambda\vdash n$.''
The Young diagram corresponding to such a partition $\lambda$ is a two-dimensional left-justified array of empty cells in which the $i^{th}$ row has $\lambda_i$ cells. See the left portion of Figure~\ref{fig:young-diagram} for an example of a Young diagram.  
A \emph{Young tableau} corresponding to a partition $\lambda$ is obtained by filling in the $n$ cells of the Young diagram with the elements of $[n]$, using each element exactly once, where the ordering within rows of the Young diagram is irrelevant.

\bigskip 

\begin{center}
\begin{figure}[ht]
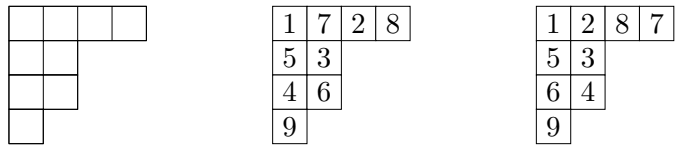

\quad
\quad
\quad
\quad 
\quad
\quad
\quad
\quad
\quad
\quad 
\yng(4,2,2,1) 
\quad
\quad
\quad
\quad 
\young(1728,53,46,9) 
\quad \quad \quad \quad
\young(1287,53,64,9) 
\caption{On the left is the Young diagram for the partition $\lambda = (4,2,2,1)$.  The middle and right present two equivalent Young tableaus for $(\{1,7,2,8\}, \{5,3\}, \{4,6\}, \{9\})$.}
\label{fig:young-diagram}
\end{figure}
\end{center}

For each partition $\lambda=(\lambda_1,\dots,\lambda_k)$ of $n$, there is an associated representation, denoted $\tau_{\lambda}$, which we now define.  Let $N_{\lambda} = {n \choose \lambda_1,\dots,\lambda_k}$ be the number of Young tableaus corresponding to partition $\lambda$, and let $\mathsf{Y}_{\lambda,1}, \ldots, \mathsf{Y}_{\lambda,N_{\lambda}}$ be an enumeration of these tableaus in some order.

\begin{definition}\label{def:replambda}
The \emph{permutation representation $\tau_\lambda$ corresponding to $\lambda$} is defined as follows:  For each $g \in \mathbb{S}_n$, $\tau_\lambda(g)$ is the $N_\lambda \times N_\lambda$ matrix (where we view rows and columns as indexed by Young tableaus corresponding to $\lambda$) which has $\tau_{\lambda}(g)(i,j)=1$ iff $\mathsf{Y}_{\lambda,i}$ maps to $\mathsf{Y}_{\lambda,j}$ under the action of $g$.
\end{definition}
 It is easy to check that $\tau_{\lambda} : \mathbb{S}_n \rightarrow \mathbb{C}^{N_{\lambda} \times N_{\lambda}}$ as defined above is indeed a representation. In fact, since the range of $\tau_{\lambda}$ is always a permutation matrix, $\tau_{\lambda}$ is also a unitary representation.
 
It turns that for $\lambda \not = (n)$, the permutation representation $\tau_{\lambda}$ is not an irreducible representation. However, it also turns out that all of the irreducible representations of $\mathbb{S}_n$ can be obtained from the permutation representations.  To explain this, we need to define a partial order over partitions of $n$:

\begin{definition}~\label{def:dominance-order}
For two partitions $\lambda$ and $\mu$ of $n$, we say that \emph{$\lambda$ dominates $\mu$}, written $\lambda$ $\rhd \mu$, if $\sum_{j \le i} \lambda_j \ge \sum_{j \le i} \mu_j$ for all $i >0$. The partial order defined by $\rhd$ is said to be the \emph{dominance order} over the partitions (equivalently, Young diagrams) of $n$. 
\end{definition}
\begin{center}
\begin{figure}[ht]
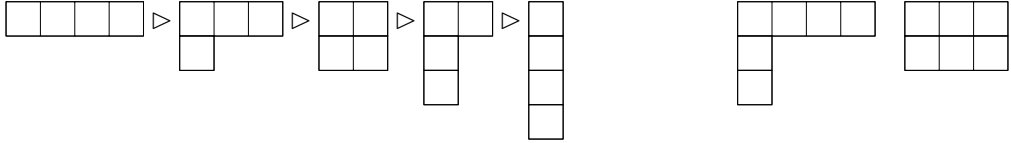

\[
\yng(4)  \rhd 
\yng(3,1)  \rhd  \yng(2,2) \rhd \yng(2,1,1) \rhd \yng(1,1,1,1)
\quad \quad \quad \quad \quad \quad
\yng(4,1,1) \quad \yng(3,3)
\]
\label{fig:young-diagram-order}
\caption{The left part of the picture depicts the dominance order across the partitions of 4; it happens to be the case that the dominance order is a total order across the partitions of 4.  This is not true in general; as depicted on the right, the two partitions $(4,1,1)$ and $(3,3)$ of 6 are incomparable under the dominance order.}
\end{figure}
\end{center}

The next result explains how the irreducible representations of $\mathbb{S}_n$ can be obtained from the representations $\{\tau_{\lambda}\}_{\lambda \vdash n}$: 
\begin{theorem} [James submodule theorem, see e.g. Theorem~3.34 of \cite{meliot2017representation}] \label{thm:James-submodule}
The irreducible representations of $\mathbb{S}_n$ are in one-to-one correspondence with the partitions $\lambda \vdash n$; we denote the irreducible representation corresponding to $\lambda$ by $\rho_{\lambda}$. In particular, when $\lambda=(n)$, then $\rho_{\lambda}$ is the trivial irreducible representation (which maps each $g \in G$ to 1).
Moreover, each permutation representation $\tau_\lambda$ is a direct sum of irreducible representations corresponding to partitions which dominate $\lambda$, i.e.
\[
\tau_{\lambda} = \mathop{\oplus}_{\mu \rhd \lambda} \mathop{\oplus}_{\ell=1}^{K_{\lambda, \mu}} \rho_{\mu}.
\]
Here the $K_{\lambda, \mu}$'s are non-negative integers, known as the \emph{Kostka numbers}, which are such that $K_{\lambda, \lambda}=1$. 
\end{theorem}

\subsubsection{Restrictions of irreducible representations} 
Fix $\lambda\vdash n$ and consider the irreducible representation $\rho_\lambda$ of $\mathbb{S}_n$. For any $m \leq n$,  $\mathbb{S}_m$ can be viewed as the  subgroup of $\mathbb{S}_n$ where elements $\{m+1,\dots,n\}$ are fixed. Hence $\rho_{\lambda}$ can also be viewed as a representation of $\mathbb{S}_n$; this representation of $\mathbb{S}_m$ is written $\rho^m_\lambda$ and is called the \emph{restriction} of $\rho_\lambda$ to $\mathbb{S}_m$. Note that $\rho^m_\lambda$ may not be an \emph{irreducible} representation of $\mathbb{S}_m$.  By~\Cref{thm:Maschke}, we have that $\rho^m_\lambda$ is equivalent to some direct sum

\[
\ignore{\rho_{\lambda} \cong }\mathop{\oplus}_{\mu \vdash m} M_{\lambda, \mu} \rho_{\mu},
\]
in which there are $M_{\lambda,\mu}$ many copies of $p_\mu$, for some non-negative integers $M_{\lambda, \mu}$. These integers are given by the so-called ``branching rule" on Young's lattice, which we now describe. 

\begin{definition}~\label{def:Young-lattice}
\emph{Young's lattice} is the partially ordered set of Young diagrams in which the partial order is given by inclusion in the following sense: given partitions $\mu$ and $\lambda$, we write ``$\mu \uparrow \lambda$'' if $\lambda$ can be obtained by adding one box to $\mu$ (in such a way that $\lambda$ is a valid partition, of course).   If there are partitions $\mu^1,\dots,\mu^r$ such that
$\mu^1 \uparrow \mu^2 \uparrow \cdots \uparrow \mu^r$, we write ``$\mu^1 \Uparrow \mu^r$.''
\end{definition}

It is convenient to draw Young's lattice in such a way that the  $n$-th level contains all and only the Young diagrams with $n$ boxes. The diagram in~\Cref{fig:young} depicts the first five levels of Young's lattice. 

\begin{figure}
\begin{center}
\includegraphics[scale=0.8]{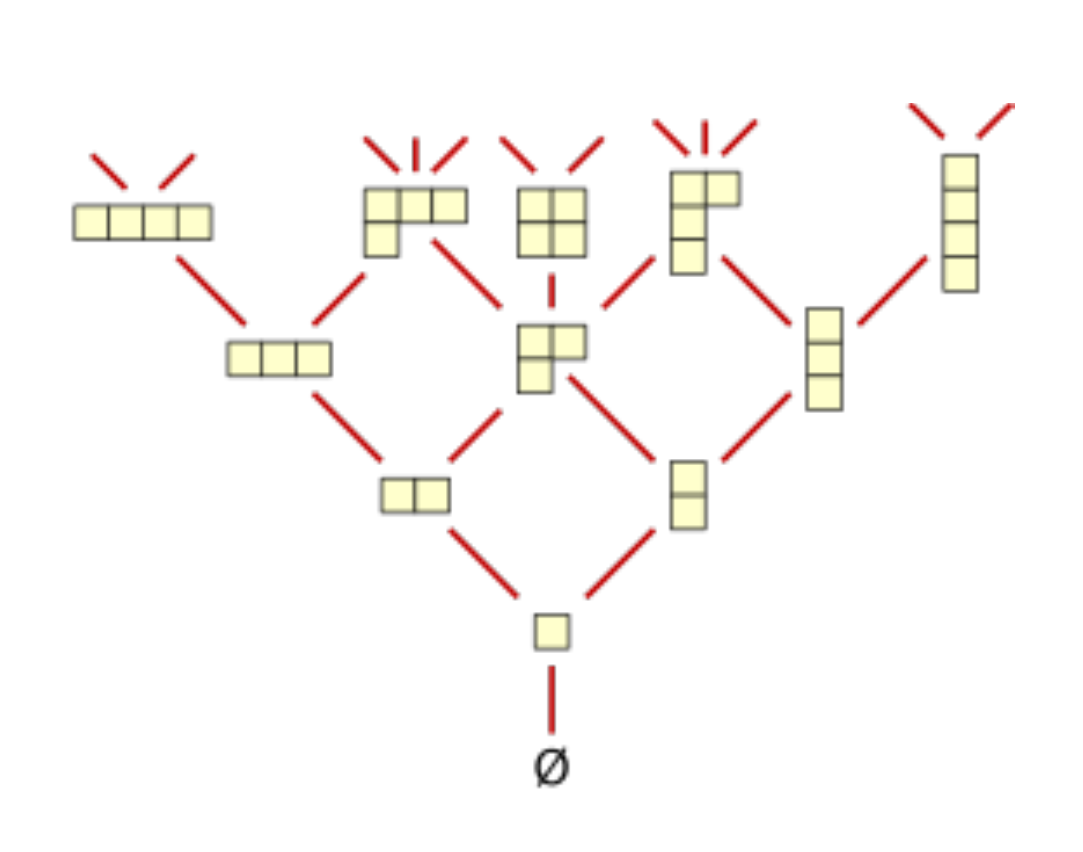}
\caption{The first five levels of Young's lattice.}
\label{fig:young}
\end{center}
\end{figure}

The next result, known as the ``branching rule,'' states that for $\lambda \vdash n$, $\rho_{\lambda}$ 
splits into a direct sum of $\rho_{\mu}$ over all $\mu  \uparrow \lambda$ when $\rho_{\lambda}$ is restricted to $\mathbb{S}_{n-1}$:

\begin{lemma} [Branching rule] \label{lem:branch1}
Let $\lambda$ be a partition of $n$ and let $\rho_{\lambda}$ be the corresponding irreducible representation of $\mathbb{S}_n$. Then
$\rho_{\lambda}^{n-1}$, the restriction of $\rho_\lambda$ to $\mathbb{S}_{n-1}$, is equivalent to
\[ 
\mathop{\oplus}_{\mu \vdash n-1 \, : \, \mu\uparrow \lambda} \rho_{\mu}.
\]
\end{lemma}
By applying~\Cref{lem:branch1} inductively we get a complete description of how $\rho_{\lambda}$ splits when it is restricted to any $\mathbb{S}_m$, $m < n$:

\begin{theorem}~\label{thm:branch}
Let $\lambda \vdash n$ and let $\rho_{\lambda}$ be the corresponding irreducible representation of $\mathbb{S}_n$. For $m < n$ we have that $\rho_{\lambda}^{m},$ the restriction of $\rho_{\lambda}$ to $\mathbb{S}_{m}$, is equivalent to
\[
 \mathop{\oplus}_{\mu \vdash m} \paths (\mu, \lambda) \rho_{\mu},
\]
where $\paths(\mu, \lambda)$ denotes the number of paths in Young's lattice from $\mu$ to $\lambda$. 
\end{theorem}

\noindent
{\bf Irreducible characters of the symmetric group.} Finally, we recall the following fundamental fact (which is a consequence, e.g., of the Murnaghan-Nakayama rule) which we will use:

\begin{fact}~\label{fact:rational}
[see e.g.~Theorem~3.10 in \cite{meliot2017representation}]
Let $\chi: \mathbb{S}_m \rightarrow \mathbb{C}$ be a character of $\mathbb{S}_m$. Then in fact $\chi$ is $\mathbb{Q}$-valued. 
\end{fact}

\subsubsection*{Acknowledgments}
We thank Mike Saks for allowing us to include his proof of Claim~\ref{clm:WY-n-ary} here. We also thank  Vic Reiner and Yuval Roichman for answering several  questions about representation theory. Anindya is grateful to Aravindan Vijayaraghavan for many useful discussions about ranking models. 
\bibliography{allrefs}{}

\begin{thebibliography}{MPPB07}

\bibitem[ABSV14]{awasthi2014learning}
P.~Awasthi, A.~Blum, O.~Sheffet, and A.~Vijayaraghavan.
\newblock Learning mixtures of ranking models.
\newblock In {\em {Advances in Neural Information Processing Systems}}, pages
  2609--2617, 2014.

\bibitem[BM08]{braverman2008noisy}
M.~Braverman and E.~Mossel.
\newblock Noisy sorting without resampling.
\newblock In {\em {Proceedings of the nineteenth annual ACM-SIAM symposium on
  Discrete algorithms}}, pages 268--276, 2008.

\bibitem[BOB07]{busse2007cluster}
L.~Busse, P.~Orbanz, and J.~Buhmann.
\newblock Cluster analysis of heterogeneous rank data.
\newblock In {\em {Proceedings of the 24th ICML}}, pages 113--120, 2007.

\bibitem[Cho94]{Choi94}
K.~P. Choi.
\newblock {On the medians of gamma distributions and an equation of Ramanujan}.
\newblock {\em Proc. Amer. Math. Soc.}, 121:245--251, 1994.

\bibitem[CR66]{curtis1966representation}
C.~Curtis and I.~Reiner.
\newblock {\em Representation theory of finite groups and associative
  algebras}, volume 356.
\newblock {American Mathematical Society}, 1966.

\bibitem[DH92]{DH92}
Persi Diaconis and Phil Hanlon.
\newblock {Eigen Analysis for Some Examples of the Metropolis Algorithm}.
\newblock {\em Contemporary Mathematics}, 138:99--117, 1992.

\bibitem[Dia88a]{diaconis1988group}
P.~Diaconis.
\newblock Group representations in probability and statistics.
\newblock {\em {Lecture Notes-Monograph Series}}, 11:i--192, 1988.

\bibitem[Dia88b]{diaconis-chap6}
Persi Diaconis.
\newblock {\em Chapter 6: Metrics on Groups, and Their Statistical Uses},
  volume Volume 11 of {\em Lecture Notes--Monograph Series}, pages 102--130.
\newblock Institute of Mathematical Statistics, 1988.

\bibitem[DS81]{DS81}
Persi Diaconis and Mehrdad Shahshahani.
\newblock {Generating a Random Permutation with Random Transpositions}.
\newblock {\em Z. Wahrscheinlichkeitstheorie verw. Gebiete}, 57:159--179, 1981.

\bibitem[DS98]{DSC98}
Persi Diaconis and Laurent Saloff{-}Coste.
\newblock What do we know about the metropolis algorithm?
\newblock {\em J. Comput. Syst. Sci.}, 57(1):20--36, 1998.

\bibitem[DST16]{de2016noisy}
A.~De, M.~Saks, and S.~Tang.
\newblock Noisy population recovery in polynomial time.
\newblock In {\em 2016 Foundations of Computer Science}, pages 675--684. IEEE,
  2016.

\bibitem[Ewe72]{Ewens72}
W.~Ewens.
\newblock The sampling theory of selectively neutral alleles.
\newblock {\em Theoretical Population Biology}, 3:87--112, 1972.

\bibitem[FV86]{fligner1986distance}
M.~Fligner and J.~Verducci.
\newblock Distance based ranking models.
\newblock {\em {Journal of the Royal Statistical Society. Series B
  (Methodological)}}, pages 359--369, 1986.

\bibitem[GJ79]{GareyJohnson79}
Michael~R. Garey and David~S. Johnson.
\newblock {\em {Computers and Intractability: A Guide to the Theory of
  NP-Completeness}}.
\newblock W. H. Freeman, 1979.

\bibitem[GP18]{gladkich2018}
A.~Gladkich and R.~Peled.
\newblock On the cycle structure of mallows permutations.
\newblock 46(2):1114--1169, 03 2018.

\bibitem[GW10]{Wigderson:representationtheory}
B.~Green and A.~Wigderson.
\newblock {Lecture notes for the 22nd McGill Invitational Workshop on
  Computational Complexity}.
\newblock 2010.

\bibitem[Jam06]{james2006representation}
Gordon~Douglas James.
\newblock {\em The representation theory of the symmetric groups}, volume 682.
\newblock Springer, 2006.

\bibitem[JV18]{jiao2018kendall}
Y.~Jiao and J.~Vert.
\newblock {The Kendall and Mallows kernels for permutations}.
\newblock {\em {IEEE transactions on pattern analysis and machine
  intelligence}}, 40(7):1755--1769, 2018.

\bibitem[KB10]{KB10}
R.~Kondor and M.~Barbosa.
\newblock {Ranking with Kernels in Fourier space}.
\newblock In {\em {COLT} 2010}, pages 451--463, 2010.

\bibitem[KL02]{kondor2002diffusion}
R.~Kondor and J.~Lafferty.
\newblock Diffusion kernels on graphs and other discrete structures.
\newblock In {\em Machine Learning, Proceedings of the 19th International
  Conference (ICML 2002)}, 2002.

\bibitem[KV10]{kumar2010generalized}
R.~Kumar and S.~Vassilvitskii.
\newblock Generalized distances between rankings.
\newblock In {\em {WWW}}, pages 571--580, 2010.

\bibitem[LB11]{lu2011learning}
T.~Lu and C.~Boutilier.
\newblock {Learning Mallows models with pairwise preferences}.
\newblock In {\em Proceedings of the 28th {ICML}}, pages 145--152, 2011.

\bibitem[LL02]{lebanon2002cranking}
G.~Lebanon and J.~Lafferty.
\newblock Cranking: Combining rankings using conditional probability models on
  permutations.
\newblock In {\em {Proceedings of the Nineteenth International Conference on
  Machine Learning}}, pages 363--370, 2002.

\bibitem[LM18]{liu2018}
A.~Liu and A.~Moitra.
\newblock {Efficiently Learning Mixtures of Mallows Models}.
\newblock In {\em Proceedings of FOCS, 2018}, 2018.

\bibitem[LZ15]{lovett2015improved}
Shachar Lovett and Jiapeng Zhang.
\newblock {Improved noisy population recovery, and reverse Bonami-Beckner
  inequality for sparse functions}.
\newblock In {\em Proceedings of the 47th Annual ACM Symposium on Theory of
  Computing}, pages 137--142, 2015.

\bibitem[Mal57]{mallows1957non}
C.~Mallows.
\newblock {Non-null ranking models. I}.
\newblock {\em Biometrika}, 44(1/2):114--130, 1957.

\bibitem[Mar14]{marden2014analyzing}
J.~Marden.
\newblock {\em Analyzing and modeling rank data}.
\newblock Chapman and Hall/CRC, 2014.

\bibitem[MC10]{meilua2010dirichlet}
M.~Meil{\u{a}} and H.~Chen.
\newblock Dirichlet process mixtures of generalized mallows models.
\newblock In {\em {Proceedings of the Twenty-Sixth Conference on Uncertainty in
  Artificial Intelligence}}, pages 358--367, 2010.

\bibitem[M{\'e}l17]{meliot2017representation}
P.~M{\'e}liot.
\newblock {\em Representation theory of symmetric groups}.
\newblock {Chapman and Hall/CRC}, 2017.

\bibitem[MM03]{murphy2003mixtures}
T.~Murphy and D.~Martin.
\newblock Mixtures of distance-based models for ranking data.
\newblock {\em Computational statistics \& data analysis}, 41(3-4):645--655,
  2003.

\bibitem[MM09]{mandhani2009tractable}
B.~Mandhani and M.~Meila.
\newblock Tractable search for learning exponential models of rankings.
\newblock In {\em {Artificial Intelligence and Statistics}}, pages 392--399,
  2009.

\bibitem[MPPB07]{meilua2007consensus}
M.~Meil{\u{a}}, K.~Phadnis, A.~Patterson, and J.~Bilmes.
\newblock Consensus ranking under the exponential model.
\newblock In {\em {Proceedings of the Twenty-Third Conference on Uncertainty in
  Artificial Intelligence}}, pages 285--294, 2007.

\bibitem[MS13]{moitra2013polynomial}
Ankur Moitra and Michael Saks.
\newblock A polynomial time algorithm for lossy population recovery.
\newblock In {\em 2013 Foundations of Computer Science}, pages 110--116. IEEE,
  2013.

\bibitem[Muk16]{mukherjee2016estimation}
S.~Mukherjee.
\newblock Estimation in exponential families on permutations.
\newblock {\em {The Annals of Statistics}}, 44(2):853--875, 2016.

\bibitem[Sak18]{Saks:18comm}
M.~Saks.
\newblock Personal communication, 2018.

\bibitem[Sta99]{StanleyECv2}
Richard~P. Stanley.
\newblock {\em {Enumerative Combinatorics: Volume 2}}.
\newblock Cambridge University Press, 1999.

\bibitem[Ste77]{stewart1977}
G.~W. Stewart.
\newblock {On the Perturbation of Pseudo-Inverses, Projections and Linear Least
  Squares Problems}.
\newblock {\em SIAM Review}, 19(4):634--662, 1977.

\bibitem[WY12]{WY12}
Avi Wigderson and Amir Yehudayoff.
\newblock {Population Recovery and Partial Identification}.
\newblock In {\em {53rd Annual {IEEE} Symposium on Foundations of Computer
  Science}}, pages 390--399, 2012.

\end{thebibliography}
\bibliographystyle{alpha}

\end{document}